\documentclass{article}
\usepackage[margin=1.0in]{geometry}

\usepackage{times}
\usepackage{graphicx} 
\usepackage{subfig}

\usepackage{natbib}

\usepackage{algorithm}
\usepackage{algorithmic}

\usepackage{hyperref}

\usepackage{amsmath}
\usepackage{amsthm}
\usepackage{chngcntr}
\usepackage{apptools}
\newtheorem{theorem}{Theorem} 
\newtheorem{lemma}{Lemma} 
\newtheorem{remark}{Remark} 
\newtheorem{definition}{Definition} 
\newtheorem{infthm}{Informal Statement of Theorem} 
\usepackage{amsmath}
\makeatletter
\def\maketag@@@#1{\hbox{\m@th\normalfont\normalsize#1}}
\makeatother
\usepackage{amsfonts}
\usepackage{amstext}
\usepackage{amssymb}
\usepackage{float}
\usepackage{color}
\usepackage{enumerate}
\usepackage{setspace}
\usepackage{multirow}
\usepackage{authblk}
\usepackage{tabu}
\newcommand{\R}{\mathbb{R}}
\newcommand{\E}{\mathbb{E}}

\begin{document} 
\title{Stochastic modified equations and adaptive stochastic gradient algorithms}

\author[1]{Qianxiao Li\thanks{qianxiao@ihpc.a-star.edu.sg}}
\author[3,4]{Cheng Tai\thanks{chengt@pku.edu.cn}}
\author[2,3,4]{Weinan E\thanks{weinan@math.princeton.edu}}
\affil[1]{Institute of High Performance Computing, Singapore}
\affil[2]{Department of Mathematics and PACM, Princeton University, USA}
\affil[3]{Beijing Institute of Big Data Research, Beijing, China}
\affil[4]{Peking University, Beijing, China}

\maketitle









\begin{abstract} 
We develop the method of stochastic modified equations (SME), in which stochastic gradient algorithms are approximated in the weak sense by continuous-time stochastic differential equations. We exploit the continuous formulation together with optimal control theory to derive novel adaptive hyper-parameter adjustment policies. Our algorithms have competitive performance with the added benefit of being robust to varying models and datasets. This provides a general methodology for the analysis and design of stochastic gradient algorithms. 
\end{abstract} 

\section{Introduction}

\label{sec:introduction}
Stochastic gradient algorithms are often used to solve optimization problems of the form
\begin{equation}
\min_{x\in \R^d}\quad  f(x) := \frac{1}{n}\sum_{i=1}^nf_i(x),
\label{eq:finite_form}
\end{equation}
where $f,f_i:\mathbb{R}^d\rightarrow\mathbb{R}$ for $i=1,\dots,n$. In machine learning applications, $f$ is typically the total loss function whereas each $f_i$ represents the loss due to the $i^{\text{th}}$ training sample. $x$ is a vector of trainable parameters and $n$ is the training sample size, which is typically very large.

Solving~\eqref{eq:finite_form} using the standard gradient descent (GD) requires $n$ gradient evaluations per step and 
is prohibitively expensive when $n\gg 1$. An alternative, the stochastic gradient descent (SGD), is to replace 
the full gradient $\nabla f$ by a sampled version, serving as an unbiased estimator. In its simplest form, the SGD iteration is written as
\begin{equation}
x_{k+1} = x_{k} - \eta \nabla f_{\gamma_k}(x_k), \label{eq:sga_iter}
\end{equation}
where $k\geq 0$ and $\{\gamma_k\}$ are i.i.d uniform variates taking values in $\{1,2,\cdots, n\}$. The step-size $\eta$ is the learning rate. Unlike GD, SGD samples the full gradient and its computational complexity per iterate is independent of $n$. For this reason, stochastic gradient algorithms have become increasingly popular in large scale problems. 

Many convergence results are available for SGD and its variants. However, most are upper-bound type results for (strongly) convex objectives, often lacking the precision and generality to characterize the behavior of algorithms in practical settings. This makes it harder to translate theoretical understanding into algorithm analysis and design. 

In this work, we address this by pursuing a different analytical direction. We derive continuous-time stochastic differential equations (SDE) that can be understood as weak approximations (i.e. approximations in distribution) of stochastic gradient algorithms. These SDEs contain higher order terms that vanish as $\eta\rightarrow0$, but at finite and small $\eta$ they offer much needed insight of the algorithms under consideration. In this sense, our framework can be viewed as a stochastic parallel of the method of {\it modified equations} in the analysis of classical finite difference methods~\cite{noh1960difference,daly1963stability,hirt1968heuristic,warming1974modified}. For this reason, we refer to these SDEs as {\it stochastic modified equations} (SME). 
Using the SMEs, we can quantify, in a precise and general way, the leading-order dynamics of the SGD and its variants. Moreover, the continuous-time treatment allows the application of optimal control theory to study the problems of adaptive hyper-parameter adjustments. This gives rise to novel adaptive algorithms and perhaps more importantly, a general methodology for understanding and improving stochastic gradient algorithms. 

\paragraph{\bf Notation. }
We distinguish sequential and dimensional indices by writing a bracket around the latter, e.g. $x_{k,(i)}$ is the $i^\text{th}$ coordinate of the vector $x_k$, the $k^{\text{th}}$ SGD iterate. 
 
\section{Stochastic Modified Equations} 
\label{sec:sme}
We now introduce the SME approximation. Background materials on SDEs are found in Supplementary Materials (SM) B and references therein. First, rewrite the SGD iteration rule~\eqref{eq:sga_iter} as 
\begin{equation}
x_{k+1} - x_{k} =- \eta \nabla f(x_k) + \sqrt{\eta} V_k,
\label{eq:sga_rewrite}
\end{equation}
where $V_k = \sqrt{\eta}(\nabla f(x_k) -\nabla f_{\gamma_k}(x_k))$ is a $d$-dimensional random vector. Conditioned on $x_k$, $V_k$ has mean 0 and covariance matrix $\eta\Sigma(x_k)$ with \small
\begin{equation}
\Sigma(x) = \frac{1}{n} \sum_{i=1}^n (\nabla f(x) - \nabla f_i(x))(\nabla f(x) - \nabla f_i(x))^T.
\label{eq:sigma}
\end{equation}\normalsize
Now, consider the Stochastic differential equation
\begin{equation}
dX_t = b(X_t) dt + \sigma(X_t)dW_t, \quad X_0 = x_0,
\label{eq:sde_0}
\end{equation}
whose Euler discretization $X_{k+1} = X_{k} + \Delta t b(X_k) + \sqrt{\Delta t}\sigma(X_k)Z_k$, $Z_k\sim \mathcal{N}(0,I)$
resembles ~\eqref{eq:sga_rewrite} if we set $\Delta t=\eta$, $b\sim-\nabla f$ and $\sigma\sim(\eta\Sigma)^{1/2}$. Then, we would expect~\eqref{eq:sde_0} to be an approximation of~\eqref{eq:sga_iter} with the identification $t=k\eta$. It is now important to discuss the precise meaning of ``an approximation''. The noises that drive the paths of SGD and SDE are independent processes, hence we must understand approximations in the {\it weak sense}. 
\begin{definition}
	\label{def:weak_conv}
	Let $0<\eta<1$, $T>0$ and set $N=\lfloor T/\eta\rfloor$. Let $G$ denote the set of functions of polynomial growth, i.e. $g\in G$ if there exists constants $K,\kappa>0$ such that $\vert g(x) \vert < K(1+\vert x \vert^\kappa )$.
	We say that the SDE~\eqref{eq:sde_0} is an {\it order $\alpha$ weak approximation} to the SGD~\eqref{eq:sga_iter} if for every $g\in G$, there exists $C>0$, independent of $\eta$, such that for all $k=0,1,\dots,N$, 
	\begin{equation}
	\vert \E g(X_{k\eta})-\E g(x_{k}) \vert < C \eta^\alpha.\nonumber
	\end{equation}
\end{definition}
The definition above is standard in numerical analysis of SDEs~\cite{milstein1995numerical,Kloeden1992}. Intuitively, weak approximations are close to the original process not in terms of individual sample paths, but their distributions. We now state informally the approximation theorem.
\begin{infthm}
	\label{thm:sme}
	Let $T>0$ and define $\Sigma:\mathbb{R}^{d}\rightarrow\mathbb{R}^{d\times d}$ by~\eqref{eq:sigma}. Assume $f,f_i$ are Lipschitz continuous, have at most linear asymptotic growth and have sufficiently high derivatives belonging to $G$. Then, 
	\begin{enumerate}[(i)]
		\item The stochastic process $X_t$, $t\in[0,T]$ satisfying
		\begin{equation}\small
			dX_{t} = -\nabla f(X_t)dt + (\eta\Sigma(X_t))^{\frac{1}{2}}dW_t,
			\label{eq:sme_1}
		\end{equation}\normalsize
		is an order 1 weak approximation of the SGD. 
		\item The stochastic process $X_t$, $t\in[0,T]$ satisfying\small
		\begin{align}
			dX_{t} = -\nabla (f(X_t) &+ \tfrac{\eta}{4}\vert \nabla f(X_t) \vert^2)dt + (\eta\Sigma(X_t))^{\frac{1}{2}}dW_t
			\label{eq:sme_2}
		\end{align}\normalsize
		is an order 2 weak approximation of the SGD. 
	\end{enumerate}
\end{infthm}
The full statement, proof and numerical verification of Thm.~\ref{thm:sme} is given in SM. C. We hereafter call equations~\eqref{eq:sme_1} and~\eqref{eq:sme_2} {\it stochastic modified equations} (SME) for the SGD iterations~\eqref{eq:sga_iter}. We refer to the second order approximation~\eqref{eq:sme_2} for exact calculations in Sec.~\ref{sec:dynamics} whereas for simplicity, we use the first order approximation~\eqref{eq:sme_1} when discussing acceleration schemes in Sec.~\ref{sec:acc_methods}, where the order of accuracy is less important. 

Thm.~\ref{thm:sme} allows us to use the SME to deduce distributional properties of the SGD. This result differs from usual convergence studies in that it describes dynamical behavior and is derived without convexity assumptions on $f$ or $f_i$. In the next section, we use the SME to deduce some dynamical properties of the SGD. 

\section{The Dynamics of SGD}
\label{sec:dynamics}
\subsection{A Solvable SME}
\label{sec:solvable}
We start with a case where the SME is exactly solvable. 
\label{sec:dynamics_example}
Let $n=2$, $d=1$ and set $f(x)=x^2$ with $f_1(x)=(x-1)^2-1$ and $f_2(x)=(x+1)^2-1$. Then, the SME~\eqref{eq:sme_2} for the SGD iterations on this objective is (see SM. D.1)
\begin{equation*}
dX_t = -2(1+\eta) X_t dt + 2\sqrt{\eta} dW_t,
\end{equation*}
with $X_0 = x_0$. This is the well-known Ornstein-Uhlenbeck process~\cite{uhlenbeck1930theory}, which is exactly solvable (see SM. B.3), yielding the Gaussian distribution
\begin{equation}
	X_t \sim \mathcal{N} (x_0 e^{-2(1+\eta)t}, \tfrac{\eta}{1+\eta} (1-e^{-4(1+\eta)t})).\nonumber
	\label{eq:ou_dist}
\end{equation}
We observe that $\E X_t=x_0 e^{-2(1+\eta)t}$ converges exponentially to the optimum $x=0$ with rate $-2(1+\eta)$ but $\text{Var}X_t={\eta}\left(1-e^{-4(1+\eta)t}\right)/{(1+\eta)}$ increases from $0$ to an asymptotic value of $\eta/(1+\eta)$. The separation $t^*$ between the descent phase and the fluctuations phase is given by $\E X_{t^*}=\sqrt{\text{Var}X_{t^*}}$, whose solution is
\begin{equation}
t^* = \tfrac{1}{4(1+\eta)}\log(1+\tfrac{\eta+1}{\eta}x_0^2)\nonumber
\label{eq:quad_tstar}
\end{equation}
For $t<t^*$, descent dominates and when $t>t^*$, fluctuation dominates. This two-phase behavior is known for convex cases via error bounds~\cite{moulines2011non,needell2014stochastic}. Using the SME, we obtained a precise characterization of this behavior, including an exact expression for $t^*$. In Fig.~\ref{fig:quadratic_comparisons}, we verify the SME predictions regarding the mean, variance and the two-phase behavior.

\subsection{Stochastic Asymptotic Expansion}
\label{sec:asymp}
In general, we cannot expect to solve the SME exactly, especially for $d>1$. However, observe that the noise terms in the SMEs~\eqref{eq:sme_1} and~\eqref{eq:sme_2} are $\mathcal{O}(\eta^{1/2})$. Hence, we can write $X_t$ as an asymptotic series $X_t = X_{0,t} + \sqrt{\eta} X_{1,t} + \dots$
where each $X_{j,t}$ is a stochastic process with initial condition $X_{0,0}=x_0$ and $X_{j,0}=0$ for $j\geq 1$. We substitute this into the SME and expand in orders of $\eta^{1/2}$ and equate the terms of the same order to get equations for $X_{j,t}$ for $j\geq 0$. This procedure is justified rigorously in~\citet{freidlin2012random}. We obtain to leading order (see SM. B.5),
\begin{equation}
X_t \sim \mathcal{N} (X_{0,t},\eta S_t),
\label{eq:asymp}
\end{equation}
where $X_{0,t}$ solves $\dot{X}_{0,t} = -\nabla f(X_{0,t}), X_{0,0}=x_0$ and $\dot{S}_t = -S_tH_t-H_tS_t + \Sigma_t$, where $H_t=Hf(X_{0,t})$, with $Hf$ denoting the Hessian of $f$, and $\Sigma_t=\Sigma(X_{0,t})$, $S_0=0$. 
It is then possible to deduce the dynamics of the SGD. For example, there is generally a transition between descent and fluctuating regimes. $S_t$ has a steady state (assuming it is asymptotically stable) with $\vert S_\infty\vert \sim \vert \Sigma_\infty \vert/\vert H_\infty \vert$. This means that one should expect a fluctuating regime where the covariance of the SGD is of order $\mathcal{O}(\eta \vert \Sigma_\infty \vert/\vert H_\infty \vert)$. Preceding this fluctuating regime is a descent regime governed by the gradient flow. 

We validate our approximations on a non-convex objective. Set $d=2$, $n=3$ with the sample objectives $f_1(x)=x_{(1)}^2$, $f_2(x)=x_{(2)}^2$ and $f_3(x)=\delta \cos(x_{(1)}/\epsilon)\cos(x_{(2)}/\epsilon)$. 
In Fig.~\ref{fig:egg_cart_comparisons}(a), we plot $f$ for $\epsilon=0.1,\delta=0.2$, showing the complex landscape. In Fig.~\ref{fig:egg_cart_comparisons}(b), we compare the SGD moments $\vert \E(x_k) \vert$ and $\vert\text{Cov}(x_k)\vert$ with predictions of the SME and its asymptotic approximation~\eqref{eq:asymp}. We observe that our approximations indeed hold for this objective. 

\begin{figure}[t!]
	\centering
	\subfloat[]{\includegraphics[width=6cm]{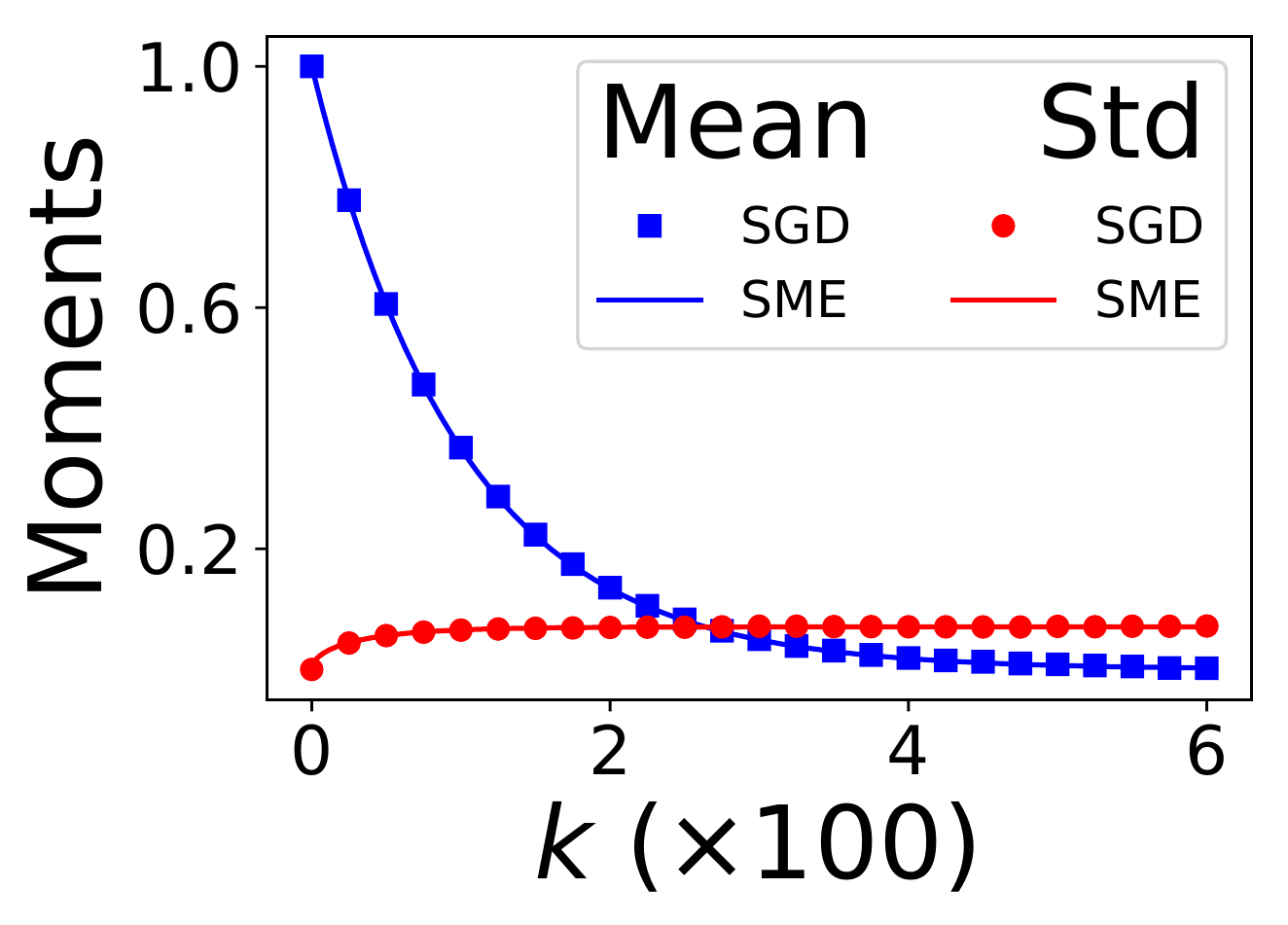}}
	\subfloat[]{\includegraphics[width=6cm]{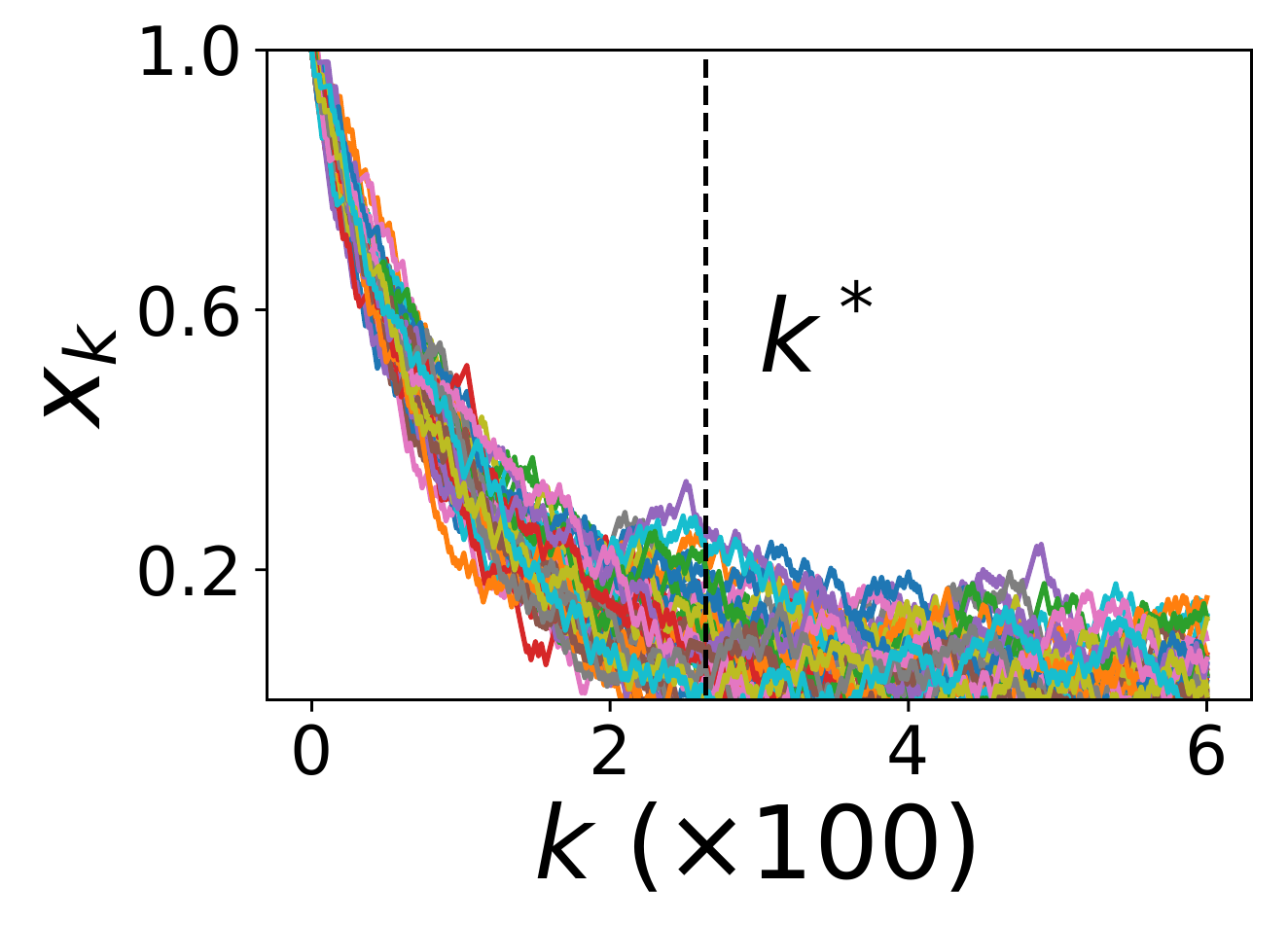}}	
	\caption{Comparison of the SME predictions vs SGD for the simple quadratic objective. We set $x_0=1$, $\eta=$5e-3. (a) The predicted mean and standard deviations agree well with the empirical moments of the SGD, obtained by averaging 5e3 runs. (b) 50 sample SGD paths the predicted transition time $k^*=t^*/\eta$. We observe that $k^*$ corresponds to the separation of descent and fluctuating regimes for typical sample paths. } 
	\label{fig:quadratic_comparisons}
\end{figure}
\begin{figure}[t!]
	\centering
	\subfloat[]{\includegraphics[height=4cm]{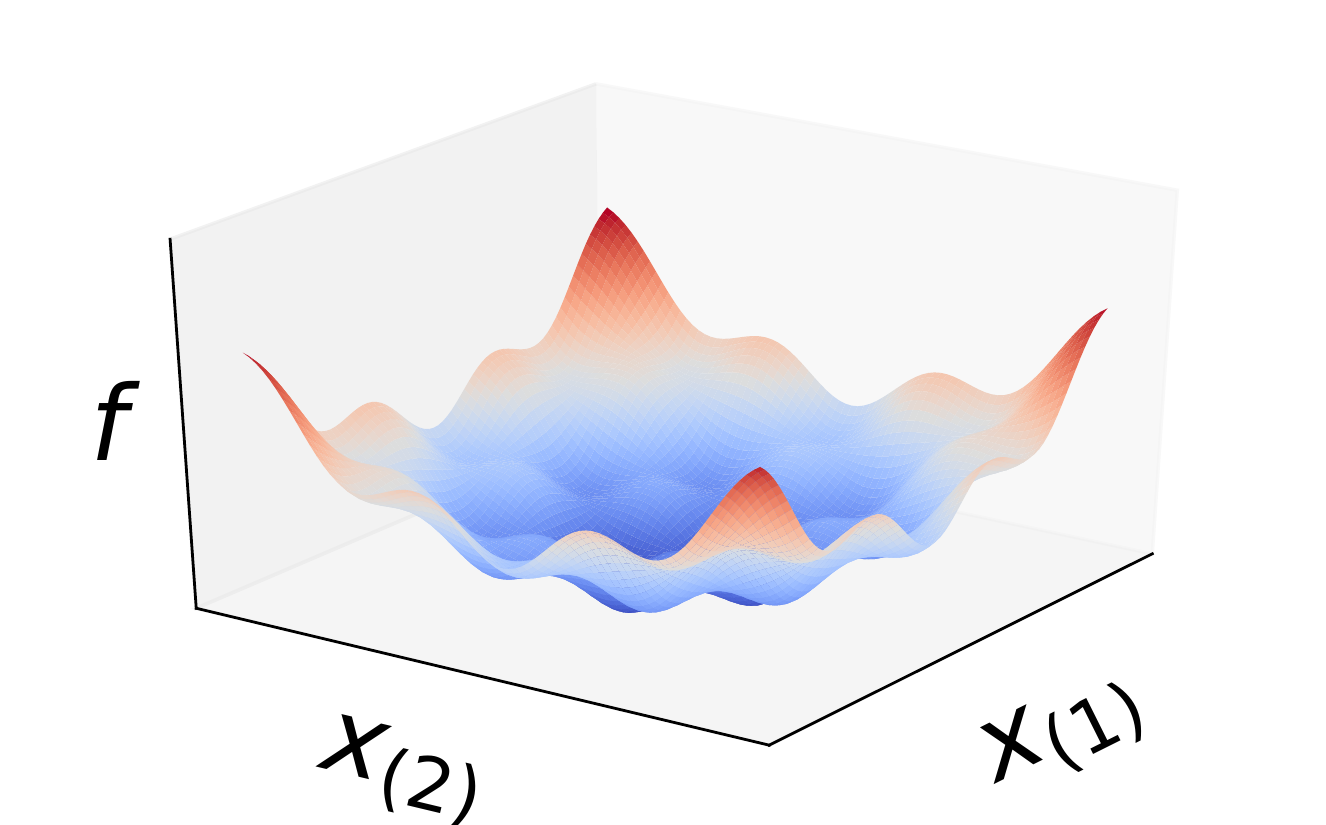}}\subfloat[]{\includegraphics[height=4cm]{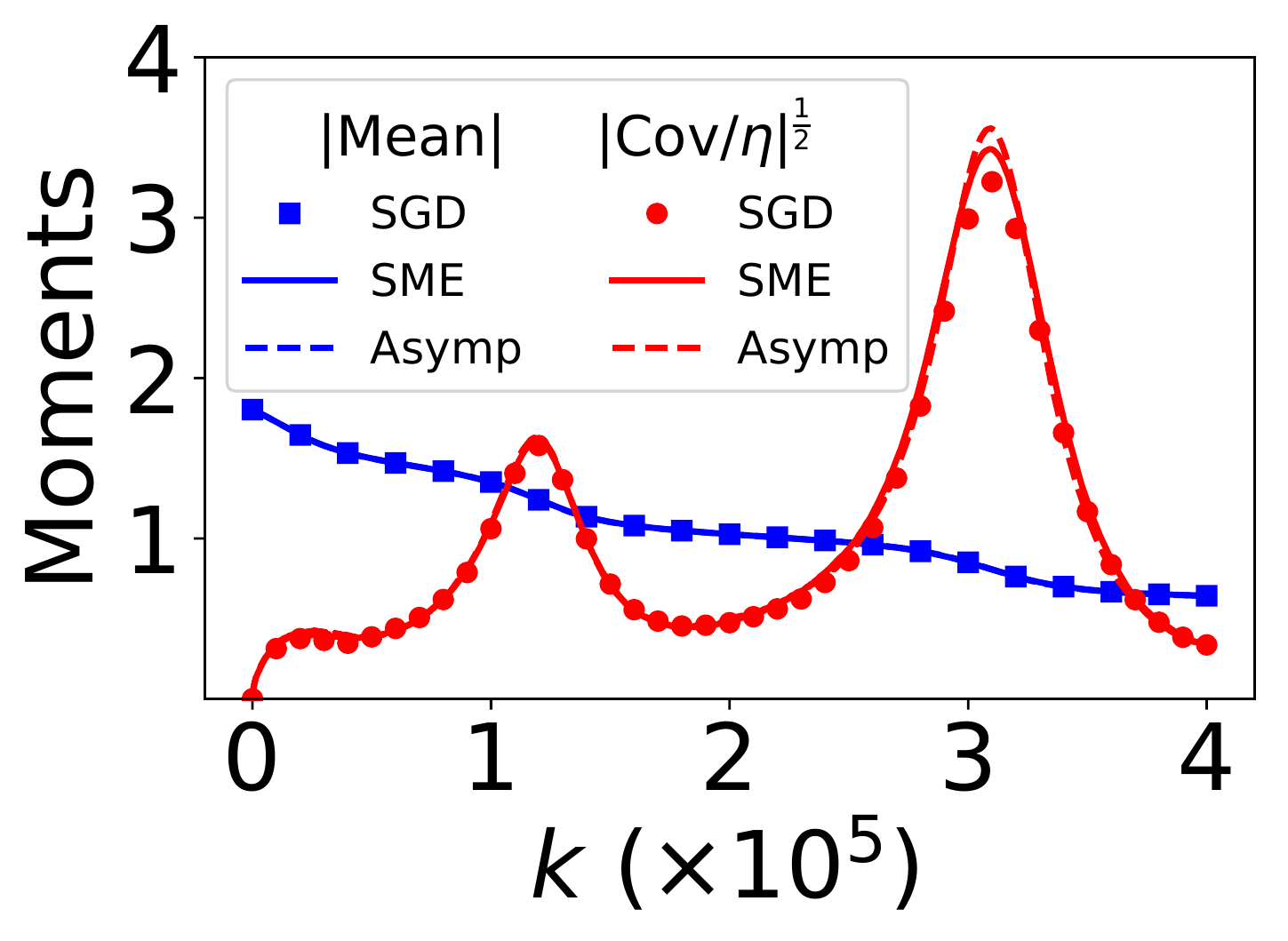}}
	\caption{Comparison of the moments of SGD iterates with the SME and its asymptotic approximation (Asymp,  Eq.~\ref{eq:asymp}) for the non-convex objective with $\delta=0.2$ and $\epsilon=0.1$. The landscape is shown in (a). In (b), we plot the magnitude of the mean and the covariance matrix for the SGD, SME and Asymp. We take $\eta=$1e-4 and $x_0=(1,1.5)$. All moments are obtained by sampling over 1e3 runs (the SME and Asymp are integrated numerically). We observe a good agreement. }
	\label{fig:egg_cart_comparisons}
\end{figure}

\section{Adaptive Hyper-parameter Adjustment}
\label{sec:acc_methods}
We showed in the previous section that the SME formulation help us better understand the precise dynamics of the SGD. The natural question is how this can translate to designing practical algorithms. In this section, we exploit the continuous-time nature of our framework to derive adaptive learning rate and momentum parameter adjustment policies. These are particular illustrations of a general methodology to analyze and improve upon SGD variants. We will focus on the one dimensional case $d=1$, and subsequently apply the results to high dimensional problems by local diagonal approximations.  

\subsection{Learning Rate}
\label{sec:control_lr}

\subsubsection{Optimal Control Formulation}
1D SGD iterations with learning rate adjustment can be written as
\begin{equation}
x_{k+1} = x_{k} - \eta u_k f^\prime(x_k), 
\label{eq:sgd_iter_lr}
\end{equation}
where $u_k\in [0,1]$ is the adjustment factor and $\eta$ is the maximum allowed learning rate. The corresponding SME for~\eqref{eq:sgd_iter_lr} is given by (SM. D.1)
\begin{equation}
dX_t = -u_t f^\prime(X_t) dt + u_t\sqrt{\eta\Sigma(X_t)}dW_t,
\label{eq:sme_iter_lr}
\end{equation}
where $u_t\in [0,1]$ is now the continuous time analogue of the adjustment factor $u_k$ with the usual identification $t=k\eta$. The effect of learning rate adjustment on the dynamics of SGD is clear. Larger $u_k$ results in a larger drift term in the SME and hence faster initial descent. However, the same factor is also multiplied to the noise term, causing greater asymptotic fluctuations. The optimal learning rate schedule must balance of these two effects. 
The problem can therefore be posed as follows: given $f,f_i$, how can we best choose a schedule or policy for adjusting the learning rate in order to minimize $\E f$ at the end of the run?
More precisely, this can be cast as an optimal control problem\footnote{See SM. E for a brief overview of optimal control theory.}
\begin{equation*}
\min_{u} \E f(X_T)  \text{ subject to } \eqref{eq:sme_iter_lr},
\label{eq:sme_lr_control_gen}
\end{equation*}
where the time-dependent function $u$ is minimized over an admissible control set to be specified. To make headway analytically, we now turn to a simple quadratic objective.

\subsubsection{Optimal Control of the Learning Rate}
\label{sec:lr_control_soln}
\label{sec:lr_opt_control}
Consider the objective $f(x) = \tfrac{1}{2} a (x-b)^2$ with $a,b\in \mathbb{R}$. Moreover, we assume the $f_i$'s are such that  $\Sigma(x) = \Sigma>0$ is a positive constant. 
The SME is then
\begin{equation}
dX_t = -a u_t  (X_t-b) dt + u_t\sqrt{\eta\Sigma}dW_t.
\label{eq:sme_iter_lr_quad}
\end{equation}
Now, assume $u$ take values in the non-random control set containing all Borel-measurable functions from $[0,T]$ to $[0,1]$. Defining $m_t=\E f(X_t)$, and applying It\^{o} formula to~\eqref{eq:sme_iter_lr_quad}, we have
\begin{equation}
\dot{m}_t = -2a u_t m_t + \tfrac{1}{2} a\eta \Sigma u_t^2.
\label{eq:sme_lr_control_ode}
\end{equation}
Hence, we may now recast the control problem as 
\begin{equation*}
\min_{u:[0,T]\rightarrow[0,1]} m_T  \text{ subject to } \eqref{eq:sme_lr_control_ode}.
\end{equation*}
This problem can solved by dynamic programming, using the Hamilton-Jacobi-Bellman equation~\cite{bellman1956dynamic}. We obtain the optimal control policy (SM. E.3)
\begin{equation}
u^*_t = \begin{cases} 
1 \qquad &a\leq0,\\ 
\min(1,\tfrac{2 m_t}{\eta\Sigma})\qquad &a>0. 
\end{cases}
\label{eq:sme_lr_control_solution}
\end{equation}
This policy is of {\it feed-back} form since it depends on the current value of the controlled variable $m_t$. 
Let us interpret the solution. First, if $a<0$ we always set the maximum learning rate $u_t=1$. This makes sense because we have a concave objective where symmetrical fluctuations about any point $x$ results in a lower average value of $f(x)$. Hence, not only do high learning rates improve descent, the high fluctuations that accompany it also lowers $\E f$. Next, For the convex case $a>0$, the solution tells us that when the objective value is large compared to variations in the gradient, we should use the maximum learning rate. When the objective decreases sufficiently, fluctuations will dominate and hence we should lower the learning rate according to the feed-back policy $u_t = {2 m_t}/{\eta\Sigma}$. 

With the policy~\eqref{eq:sme_lr_control_solution}, we can solve~\eqref{eq:sme_lr_control_ode} and plug the solution for $m_t$ back into~\eqref{eq:sme_lr_control_solution} to obtain the annealing schedule
\begin{equation*}
u^*_t = \begin{cases} 
1 \qquad &a\leq 0\text{ or }t\leq t^*,\\ 
\tfrac{1}{1+a(t-t^*)}\qquad &a>0\text{ and } t>t^*,
\end{cases}
\end{equation*}
where $t^*=(1/2a)\log({4 m_0 }/{\eta \Sigma-1})$. Note that by putting $a=2,b=0,\Sigma=4$, for small $\eta$, this expression agrees with the transition time~\eqref{eq:quad_tstar} between descent and fluctuating phases for the SGD dynamics considered in Sec.~\ref{sec:solvable}. Thus, this annealing schedule says that maximum learning rate should be used for descent phases, whereas $\sim 1/t$ decay on learning rate should be applied after onset of fluctuations. Our annealing result agree asymptotically with the commonly studied annealing schedules~\citep{moulines2011non,shamir2013stochastic}, but the difference is that we suggest maximum learning rate before the onset of fluctuations. Of course, the key limitation is that our result is only valid for this particular objective. This naturally brings us to the next question: how does one apply the optimal control results to general objectives?

\subsubsection{Application to General Objectives}
\label{sec:appl_via_localapprox}
Now, we turn to the setting where $d>1$ and $f,f_i$ are not necessarily quadratic. The most important result in Sec.~\ref{sec:lr_control_soln} is the feed-back control law~\eqref{eq:sme_lr_control_solution}. To apply it, we make a {\it local diagonal-quadratic assumption}: we assume that for each $x\in\mathbb{R}^d$, there exists $a_{(i)},b_{(i)}\in\mathbb{R}$ so that 
$f(x) \approx \tfrac{1}{2}\sum_{i=1}^{d} a_{(i)}(x_{(i)} - b_{(i)})^2$
holds locally in $x$. We also assume $\Sigma(x)\approx \text{diag}\{\Sigma_{(1)},\dots,\Sigma_{(d)}\}$ where each $\Sigma_{(i)}$ is locally constant. By considering a separate learning rate scale $u_{(i)}$ for each trainable dimension, the control problem decouples to $d$ separate problems of the form considered in Sec.~\ref{sec:lr_opt_control}. And hence, we may set $u^*_{(i)}$ element-wise according to the policy~\eqref{eq:sme_lr_control_solution}.

Since we only assume that the diagonal-quadratic assumption holds locally, the terms $a_{(i)}$, $b_{(i)}$, $\Sigma_{(i)}$ and $m_{(i)}\approx\tfrac{1}{2}a_{(i)}(x_{(i)} - b_{(i)})^2$ must be updated on the fly. There are potentially many methods for doing so. The approach we take exploits the linear relationship $\nabla f_{(i)} \approx a_{(i)}(x_{(i)} - b_{(i)})$.  Consequently, we may estimate $a_{(i)},b_{(i)}$ via linear regression on the fly: for each dimension, we maintain exponential moving averages (EMA) $\{\overline{g}_{k,(i)},\overline{g^2}_{k,(i)},\overline{x}_{k,(i)},\overline{x^2}_{k,(i)},\overline{xg}_{k,(i)}\}$ where $g_{k,(i)}=\nabla f_{\gamma_k}(x_k)_{(i)}$. For example, $\overline{g}_{k+1,(i)}=\beta_{k,(i)}\overline{g}_{k,(i)} + (1-\beta_{k,(i)})g_{k,(i)}$.
The EMA decay parameter $\beta_{k,(i)}$ controls the effective averaging window size. We adaptively adjust it so that it is small when gradient variations are large, and vice versa. We employ the heuristic
$\beta_{k+1,(i)} = {(\overline{g^2}_{k,(i)} - \overline{g}_{k,(i)}^2)}/{\overline{g^2}_{k,(i)}}$. This is similar to the approach in~\citet{schaul2013no}. 
We also clip each $\beta_{k+1,(i)}$ to $[\beta_\text{min},\beta_\text{max}]$ to improve stability. Here, we use $[0.9,0.999]$ for all experiments, but we checked that performance is insensitive to these values. 
We can now compute $a_{k,(i)},b_{k,(i)}$ by the ordinary-least-squares formula and $\Sigma_{k,(i)}$ as the variance of the gradients:
\begin{align}
a_{k,(i)} &= \frac{\overline{gx}_{k,(i)}-\overline{g}_{k,(i)}\overline{x}_{k,(i)}}{\overline{x^2}_{k,(i)}-\overline{x}_{k,(i)}^2},\nonumber\\
\quad b_{k,(i)} &= \overline{x}_{k,(i)} - \frac{\overline{g}_{k,(i)}}{a_{k,(i)}}, \nonumber\\
\Sigma_{k,(i)} &= {\overline{g^2}_{k,(i)}-\overline{g}_{k,(i)}^2}.
\label{eq:lr_running_avg_3}
\end{align}
This allows us to estimate the policy~\eqref{eq:sme_lr_control_solution} as
\begin{equation}
u^*_{k,(i)} = \begin{cases} 
1 \qquad &a_{k,(i)}\leq0,\\ 
\min(1,\tfrac{a_{k,(i)}(\overline{x}_{k,(i)}-b_{k,(i)})^2}{\eta \Sigma_{k,(i)}})\qquad &a_{k,(i)}>0. 
\end{cases}
\label{eq:sme_lr_estimated_policy}
\end{equation}
for $i=1,2,\dots,d$. Since quantities are computed from exponentially averaged sources, we should also update our learning rate policy in the same way. The algorithm is summarized in Alg.~\ref{alg:csgd}. Due to its optimal control origin, we hereafter call this algorithm the {\it controlled SGD} (cSGD)
\begin{algorithm}[bt!]
	\caption{controlled SGD (cSGD)}
	\label{alg:csgd}
	\begin{algorithmic}
		\STATE {\bfseries Hyper-parameters:} $\eta$, $u_0$
		\STATE Initialize $x_0$; $\beta_{0,(i)}=0.9$ $\forall i$
		\FOR{$k=0$ {\bfseries to} $(\#\text{iterations}-1)$}
		\STATE Compute sample gradient $\nabla f_{\gamma_k}(x_k)$
		\FOR{$i=1$ {\bfseries to} $d$}
		\STATE Update EMA $\{\overline{g}_{k,(i)},\overline{g^2}_{k,(i)},\overline{x}_{k,(i)},\overline{x^2}_{k,(i)},\overline{xg}_{k,(i)}\}$ with decay parameter $\beta_{k,(i)}$
		\STATE Compute $a_{k,(i)}$, $b_{k,(i)}$, $\Sigma_{k,(i)}$ using~\eqref{eq:lr_running_avg_3}
		\STATE Compute $u^*_{k,(i)}$ using~\eqref{eq:sme_lr_estimated_policy}
		\STATE $\beta_{k+1,(i)} = {(\overline{g^2}_{k,(i)} - \overline{g}_{k,(i)}^2)}/{\overline{g^2}_{k,(i)}}$ and clip
		\STATE $u_{k+1,(i)} = \beta_{k,(i)} u_{k,(i)} + (1-\beta_{k,(i)}) u^*_{k,(i)}$
		\STATE $x_{k+1,(i)} = x_{k,(i)} - \eta u_{k,(i)} \nabla f_{\gamma_k}(x_k)_{(i)}$
		\ENDFOR
		\ENDFOR
	\end{algorithmic}
\end{algorithm} 
\begin{remark}
Alg.~\ref{alg:csgd} can similarly be applied to mini-batch SGD. Let the batch-size be $M$, which reduces the covariance by $M$ times and so $\eta$ in the SME is replaced by $\eta/M$. However, at the same time estimating $\Sigma_{k}$ from mini-batch gradient sample variances will underestimate $\Sigma(x_k)$ by a factor of $M$. Thus the product $\eta\Sigma_{k}$ remains unchanged and Alg.~\ref{alg:csgd} can be applied with no changes. 
\end{remark}
\begin{remark}
	The additional overheads in cSGD are from maintaining exponential averages and estimating $a_k,b_k,\Sigma_k$ on the fly with the relevant formulas. These are $\mathcal{O}(d)$ operations and hence scalable. Our current rough implementation runs $\sim 40-60\%$ slower per epoch than the plain SGD. This is expected to be improved by optimization, parallelization or updating quantities less frequently. 
\end{remark}
\subsubsection{Performance on Benchmarks}
\label{sec:csgd_experiments}
Let us test cSGD on common deep learning benchmarks. We consider three different models. M0: a fully connected neural network with one hidden layer and ReLU activations, trained on the MNIST dataset~\citep{lecun1998mnist}; C0: a fully connected neural network with two hidden layers and Tanh activations, trained on the CIFAR-10 dataset~\citep{krizhevsky2009learning}; C1: a convolution network with four convolution layers and two fully connected layers also trained on CIFAR-10. Model details are found in SM. F.1. 
In Fig.~\ref{fig:MNIST_test_lr}, we compare the performance of cSGD with Adagrad~\citep{duchi2011adaptive} and Adam~\citep{kingma2015adam} optimizers. We illustrate in particular their sensitivity to different learning rate choices by performing a log-uniform random search over three orders of magnitude. We observe that cSGD is robust to different initial and maximum learning rates (provided the latter is big enough, e.g. we can take $\eta=1$ for all experiments) and changing network structures, while obtaining similar performance to well-tuned versions of the other methods (see also Tab.~\ref{tab:test_acc}). In particular, notice that the best learning rates found for Adagrad and Adam generally differ for different neural networks. On the other hand, many values can be used for cSGD with little performance loss. For brevity we only show the test accuracies, but the training accuracies have similar behavior (see SM. F.5). 
\begin{figure}[t!]
	\centering
	\subfloat[M0 (fully connected NN, MNIST)]{\includegraphics[width=12cm]{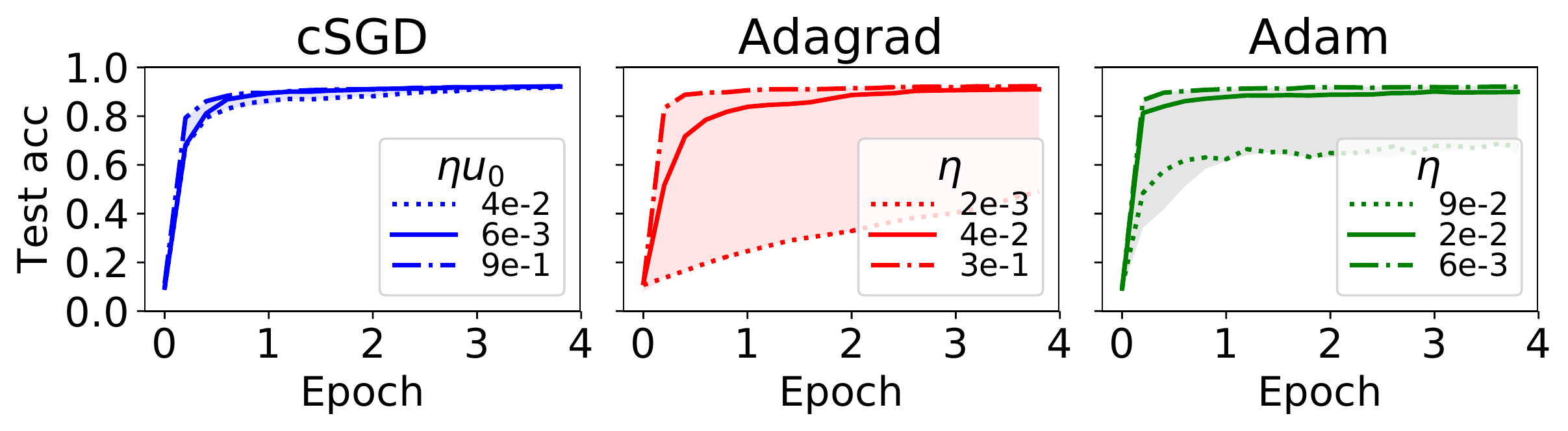}}

	\subfloat[C0 (fully connected NN, CIFAR-10)]{\includegraphics[width=12cm]{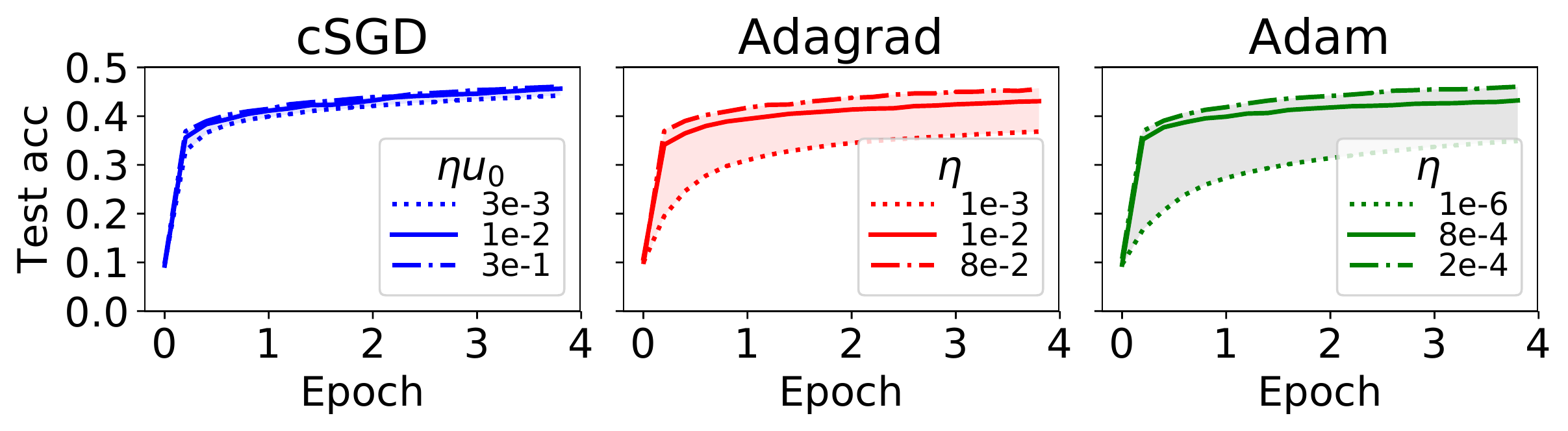}}	
	
	\subfloat[C1 (CNN, CIFAR-10)]{\includegraphics[width=12cm]{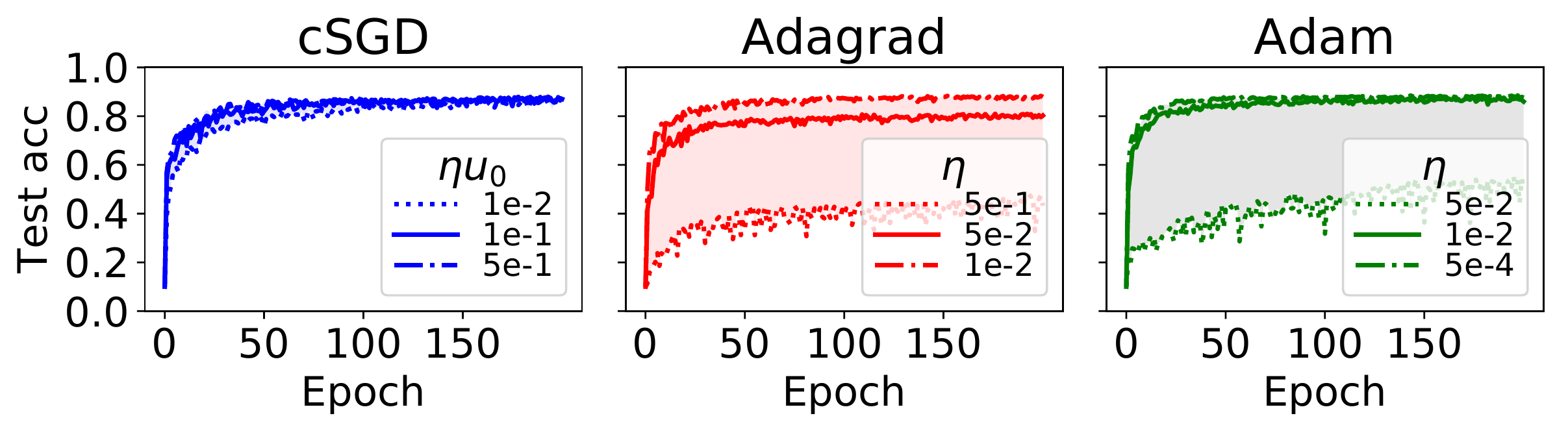}}		
	\caption{
		cSGD vs Adagrad and Adam for different models and datasets, with different hyper-parameters. For M0, we perform log-uniform random search with 50 samples over intervals: cSGD: $u_0\in$[1e-2,1], $\eta\in$[1e-1,1]; Adagrad: $\eta\in$[1e-3,1]; Adam: $\eta\in$[1e-4,1e-1]. For C0, we perform same search over intervals: cSGD: $u_0\in$[1e-2,1], $\eta\in$[1e-1,1]; Adagrad: $\eta\in$[1e-3,1]; Adam: $\eta\in$[1e-6,1e-3]. We average the resulting learning curves for each choice over 10 runs. 
		For C1, due to long training times we choose 5 representative learning rates for each method. cSGD: $\eta\in$\{1e-2,5e-2,1e-1,5e-1,1\}, $u_0=1$; Adagrad: $\eta\in$\{1e-3,5e-3,1e-2,5e-2,1e-1\}; Adam: $\eta\in$\{5e-4,1e-3,1e-2,2e-2,5e-2\}. One sample learning curve is generated for each choice.
		In all cases, we use mini-batches of size 128. We evaluate the resulting learning curves by the area-under-curve. The worst, median and best learning curves are shown as dotted, solid, and dot-dashed lines respectively. The shaded areas represent the distribution of learning curves for all searched values. 
		We observe that cSGD is relatively robust with respect to initial/maximum learning rates and the network structures, and requires little tuning while having comparable performance to well-tuned versions of the other methods (see Tab.~\ref{tab:test_acc}). This holds across different models and datasets. }
	\label{fig:MNIST_test_lr}
\end{figure}
\subsection{Momentum Parameter}
\label{sec:momentum}
Another practical way of speeding up the plain SGD is to employ momentum updates - an idea dating back to deterministic optimization~\cite{polyak1964some,nesterov1983method,qian1999momentum}. However, the stochastic version has important differences, especially in regimes where sampling noise dominates. Nevertheless, provided that the momentum parameter is well-tuned, the momentum SGD (MSGD) is very effective in speeding up convergence, particularly in early stages of training~\citep{sutskever2013importance}. 

Selecting an appropriate momentum parameter is important in practice. Typically, generic values (e.g. 0.9, 0.99) are suggested without fully elucidating their effect on the SGD dynamics. In this section, we use the SME framework to analyze the precise dynamics of MSGD and derive effective adaptive momentum parameter adjustment policies. 

\subsubsection{SME for MSGD}
The SGD with momentum can be written as the following coupled updates
\begin{align}
v_{k+1} &= \mu v_k - \eta f'_{\gamma_k}(x_k), \nonumber\\
x_{k+1} &= x_k + v_{k+1}. 
\label{eq:momentum_iter}
\end{align}
The parameter $\mu$ is the momentum parameter taking values in the range $0\leq\mu\leq 1$. Intuitively, the momentum term $v_k$ remembers past update directions and pushes along $x_k$, which may otherwise slow down at e.g. narrow parts of the landscape. 
The corresponding SME is now a coupled SDE 
\begin{align}
dV_t &= ( -{\eta}^{-1}({1-\mu}) V_t - f'(X_t) )dt + (\eta \Sigma(X_t))^{\frac{1}{2}} dW_t, \nonumber\\
dX_t &= {\eta}^{-1}V_t dt. 
\label{eq:momentum_sme}
\end{align}
This can be derived by comparing~\eqref{eq:momentum_iter} with the Euler discretization scheme of~\eqref{eq:momentum_sme} and matching moments. Details can be found in SM. D.3. 
\subsubsection{The Effect of Momentum}
\label{sec:effect_mom_quad}
As in Sec.~\ref{sec:control_lr}, we take the prototypical example $f(x)=\tfrac{1}{2}a(x-b)^2$ with $\Sigma$ constant and study the effect of incorporating momentum updates. Define $M_t=(\E f(X_t), \E V^2_t, \E V_tf'(X_t))\in \mathbb{R}^3$.
By applying It\^{o} formula to~\eqref{eq:momentum_sme}, we obtain the ODE system
\begin{align}
\dot{M}_t = A(\mu)M_t + &B,\nonumber\\
A(\mu) = \left(\begin{smallmatrix}0 & 0 & {a}/{\eta}\\0 & -{2}(1-\mu)/{\eta} & -2\\-2 & {1}/{\eta} & -(1-\mu)/{\eta}\end{smallmatrix}\right)&,
\, B = \left(\begin{smallmatrix}0\\\eta \Sigma\\0\end{smallmatrix}\right).
\label{eq:M_eqn}
\end{align}
If $a<0$, $A(\mu)$ has a positive eigenvalue and hence $M_t$ diverges exponentially. Since $f$ is negative, its value must then decrease exponentially for all $\mu$, and the descent rate is maximized at $\mu=1$. The more interesting case is when $a>0$. Instead of solving~\eqref{eq:M_eqn}, we observe that all eigenvalues of $A(\mu)$ have negative real parts as long as $\mu<1$. Therefore, $M_t$ has an exponential decay dominated by $\vert \mathcal{R}\lambda(\mu)\vert$, where $\mathcal{R}$ denotes real part and 
$\lambda(\mu) = -\tfrac{1}{\eta} [{(1-\mu) - \sqrt{(1-\mu)^2 - 4a\eta}}]$
is the eigenvalue with the least negative real part.
Observe that the descent rate $\vert \mathcal{R}\lambda(\mu)\vert$ is maximized at 
\begin{equation}
\mu_{\text{opt}} =  \max(1-2\sqrt{a\eta},0)
\label{eq:mu_opt}
\end{equation}
and when $\mu>\mu_{\text{opt}}$, $\lambda$ becomes complex. 
Also, from~\eqref{eq:M_eqn} we have $M_t\rightarrow M_{\infty} = -A(\mu)^{-1}B = \left(\begin{smallmatrix}\frac{\eta  \Sigma}{4(1-\mu) }& \frac{\eta ^2 \Sigma}{2(1-\mu) } & 0\end{smallmatrix}\right)$, provided the steady state is stable. 
The role of momentum in this problem is now clear. To leading order in $\eta$ we have $\lambda(\mu)\sim -2a/(1-\mu)$ for $\mu\leq\mu_{\text{opt}}$. Hence, any non-zero momentum will improve the initial convergence rate. In fact, the choice $\mu_\text{opt}$ is optimal and above it, oscillations set in because of a complex $\lambda$. At the same time, increasing momentum also causes increment in eventual fluctuations, since $\vert M_\infty\vert = \mathcal{O}((1-\mu)^{-1})$. In Fig.~\ref{fig:mom_compare}(a), we demonstrate the accuracy of the SME prediction~\eqref{eq:M_eqn} by comparing MSGD iterations.
\begin{figure}
	\centering
	\subfloat[]{\includegraphics[width=6cm]{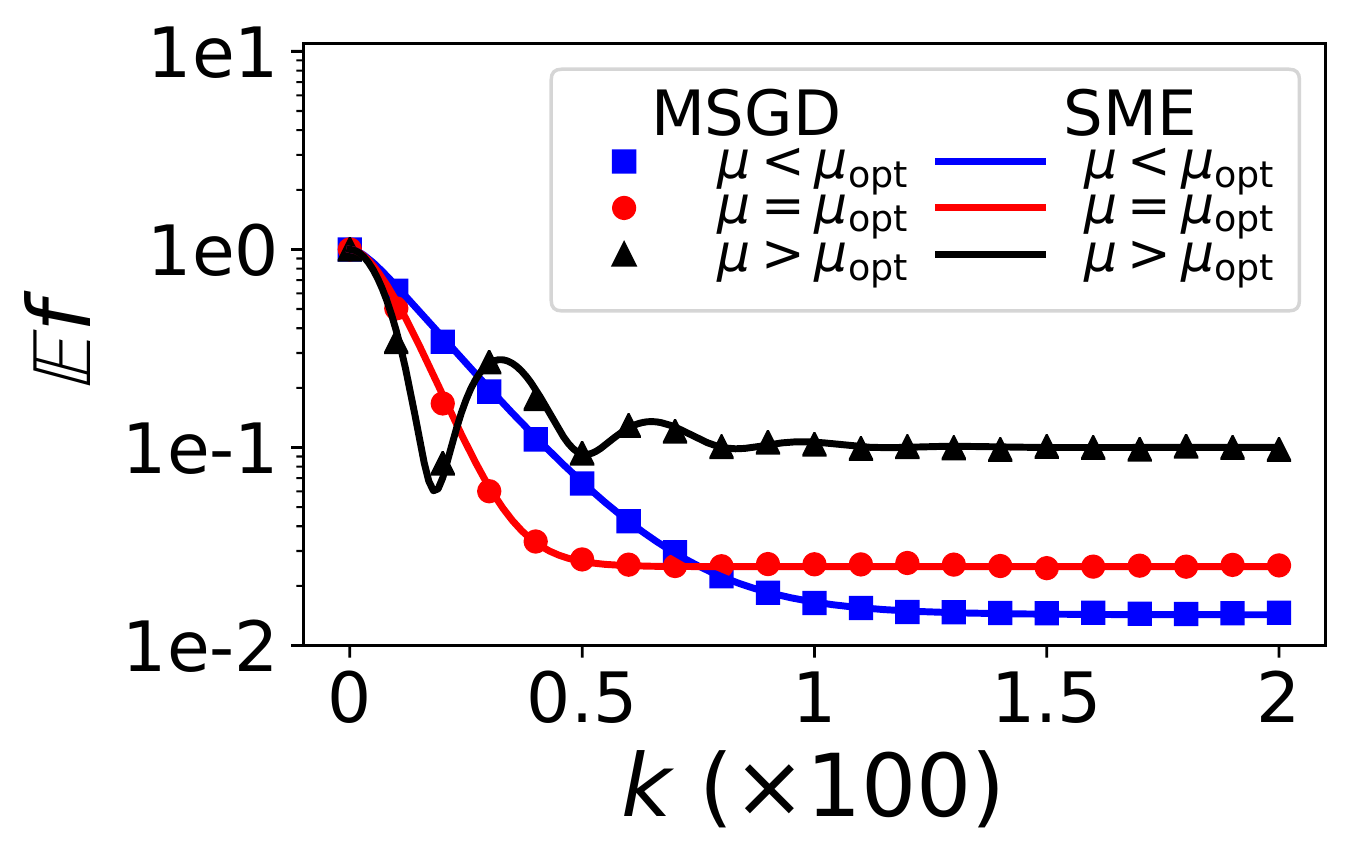}}
	\subfloat[]{\includegraphics[width=6cm]{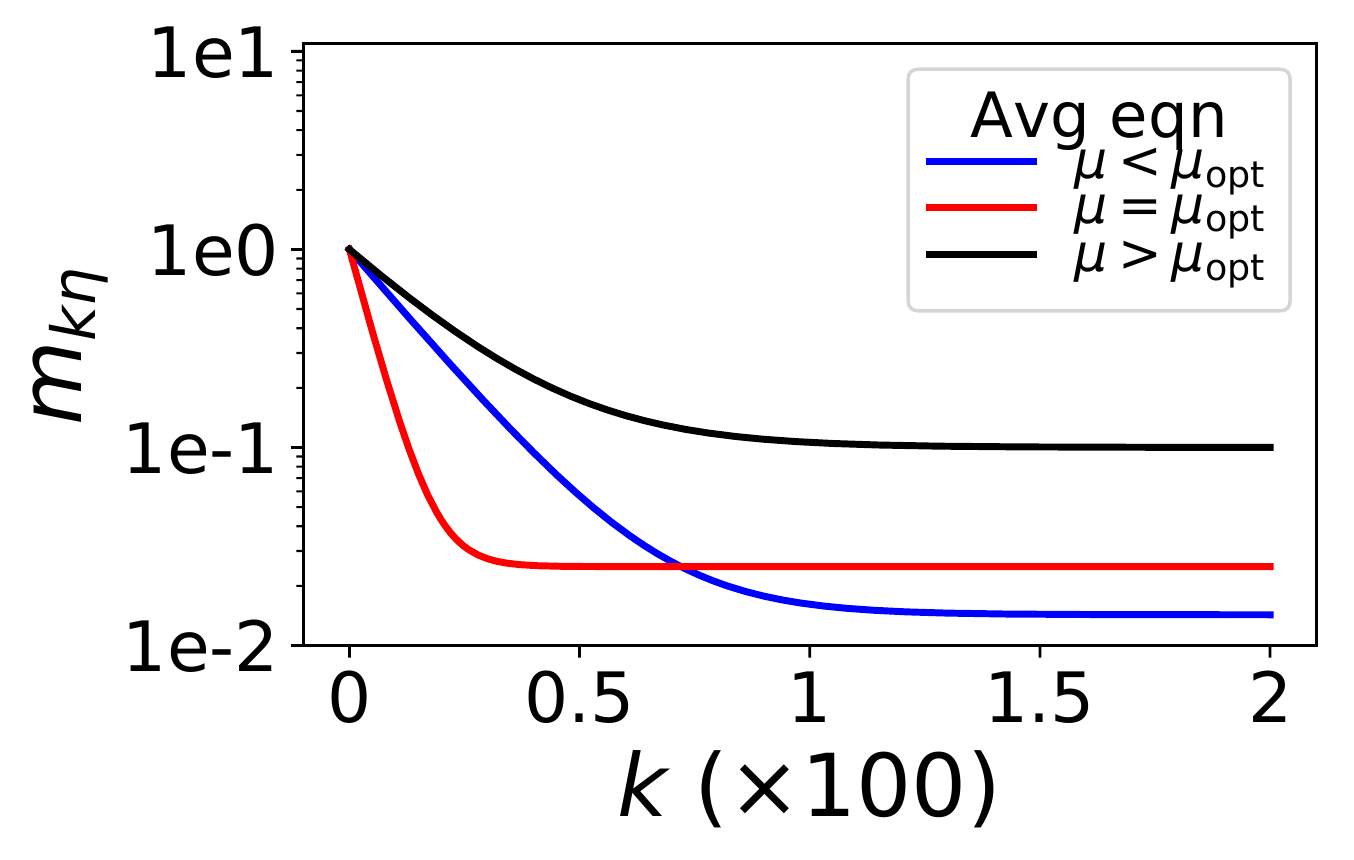}}	
	\caption{(a) Comparison of the SME prediction~\eqref{eq:M_eqn} with SGD for the same quadratic example in Sec.~\ref{sec:dynamics_example}, which has $a=2$, $b=0$ and $\Sigma=4$. We set $\eta=$5e-3 so that $\mu_{\text{opt}}=0.8$. We plot the mean of $f$ averaged over 1e5 SGD runs against the SME predictions for $\mu=0.65,0.8,0.95$. We observe that in all cases the approximation is accurate. In particular, the SME correctly predicts the effect of momentum: $\mu=\mu_{\text{opt}}$ gives the best average initial descent rate, $\mu>\mu_{\text{opt}}$ causes oscillatory behavior, and increasing $\mu$ generally increases asymptotic fluctuations. (b) The dynamics of averaged equation~\eqref{eq:m_eqn}, which serves as an approximation of the solution of the full SME moment equation~\eqref{eq:M_eqn}. }
	\label{fig:mom_compare}
\end{figure}
Armed with an intuitive understanding of the effect of momentum, we can now use optimal control to design policies to adapt the momentum parameter.

\subsubsection{Optimal Control of the Momentum Parameter}
\label{sec:momentum_opt_control}
For $a<0$, we have discussed previously that $\mu=1$ maximizes the descent rate and fluctuations generally help decrease concave functions. Thus, the optimal control is always $\mu=1$. The non-trivial case is when $a>0$. Due to its bi-linearity, directly controlling~\eqref{eq:M_eqn} leads to bang-bang type solutions\footnote{Bang-bang solutions are control solutions lying on the boundary of the control set and abruptly jumps among the boundary values. For example, in this case it jumps between $\mu=0$ and $\mu=1$ repeatedly.} that are rarely feed-back laws~\cite{pardalos2010optimization} and thus difficult to apply in practice. Instead, we notice that the descent rate is dominated by $\mathcal{R}\lambda(\mu)$, and the leading order asymptotic fluctuations is $\eta\Sigma/(4(1-\mu))$, hence we may consider
\begin{equation}
	\dot{m}_t = \mathcal{R}\lambda(\mu)(m_t - m_\infty(\mu))
	\label{eq:m_eqn}
\end{equation}
where $m_t\in\mathbb{R}$ and $m_\infty(\mu) = {\eta\Sigma}/{(4(1-\mu))}$ is the leading order estimate of $\vert M_\infty \vert$. Equation~\eqref{eq:m_eqn} can be understood as the approximate evolution, in an averaged sense, of the magnitude of $M_t$. Fig.~\ref{fig:mom_compare}(b) shows that~\eqref{eq:m_eqn} is a reasonable approximation of the dynamics of MSGD. This allows us to pose the optimal control problem on the momentum parameter as
\begin{equation*}
\min_{\mu:[0,T]\rightarrow[0,1]} m_T \text{ subject to \eqref{eq:m_eqn}},
\end{equation*}
with $\mu=\mu_t$. 
Solving this control problem yields the (approximate) feed-back policy (SM. E.4)
\begin{equation}
\mu^*_t = \begin{cases} 
1 \qquad &a\leq0,\\ 
\min(\mu_{\text{opt}},\max(0,1-\tfrac{\eta\Sigma}{4m_t}))\qquad &a>0,
\end{cases}
\label{eq:sme_mom_control_solution}
\end{equation}
with $\mu_\text{opt}$ given in~\eqref{eq:mu_opt}. This says that when far from optimum ($m_t$ large), we set $\mu=\mu_\text{opt}$ which maximizes average descent rate. When $m_t/\eta\Sigma\sim \sqrt{a\eta}$, fluctuations set in and we lower $\mu$. 

As in Sec.~\ref{sec:appl_via_localapprox}, we turn the control policy above into a generally applicable algorithm by performing local diagonal-quadratic approximations and estimating the relevant quantities on the fly. The resulting algorithm is mostly identical to Alg.~\ref{alg:csgd} except we now use \eqref{eq:sme_mom_control_solution} to update $\mu_{k,(i)}$ and SGD updates are replaced with MSGD updates (see S.M. F.4 for the full algorithm). We refer to this algorithm as the {\it controlled momentum SGD} (cMSGD). 

\subsubsection{Performance on Benchmarks}
\label{sec:cmsgd_experiments}
We apply cMSGD to the same three set-ups in Sec.~\ref{sec:csgd_experiments}, and compare its performance to the plain Momentum SGD with fixed momentum parameters (MSGD) and the annealing schedule suggested in~\cite{sutskever2013importance}, with $\mu_k=\min(1-2^{-1-\log_2(\lfloor k/ 250\rfloor +1)},\mu_\text{max})$ (MSGD-A).
In Fig.~\ref{fig:compare_cmsgd}, we perform a log-uniform search over the hyper-parameters $\mu_0$, $\mu$ and $\mu_\text{max}$. We see that cMSGD achieves superior performance to MSGD and MSGD-A (see Tab.~\ref{tab:test_acc}), especially when the latter has badly tuned $\mu,\mu_\text{max}$. Moreover, it is insensitive to the choice of initial $\mu_0$. Just like cSGD, this holds across changing network structures. Further, cMSGD also adapts to other hyper-parameter variations. In Fig.~\ref{fig:mom_lr_sens}, we take tuned $\mu,\mu_\text{max}$ (and any $\mu_0$) and vary the learning rate $\eta$. We observe that cMSGD adapts to the new learning rates whereas the performance of MSGD and MSGD-A deteriorates and $\mu,\mu_\text{max}$ must be re-tuned to obtain reasonable accuracy. In fact, it is often the case that MSGD and MSGD-A diverge when $\eta$ is large, whereas cMSGD remains stable. 
\begin{figure}[t!]
	\centering
	\subfloat[M0 (fully connected NN, MNIST)]{\includegraphics[width=12cm]{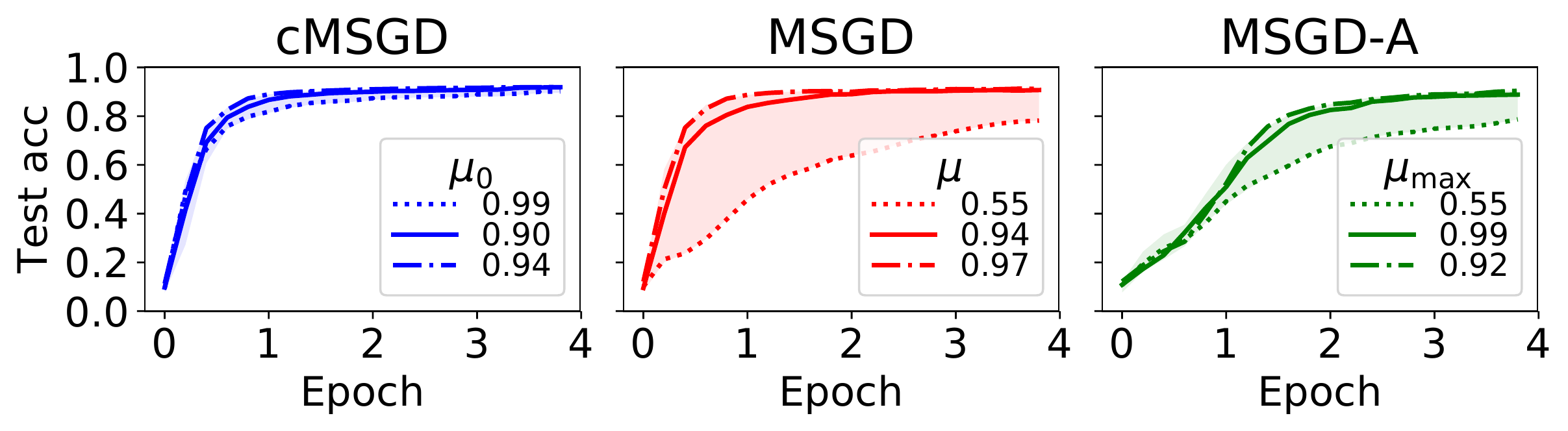}}
	
	\subfloat[C0 (fully connected NN, CIFAR-10)]{\includegraphics[width=12cm]{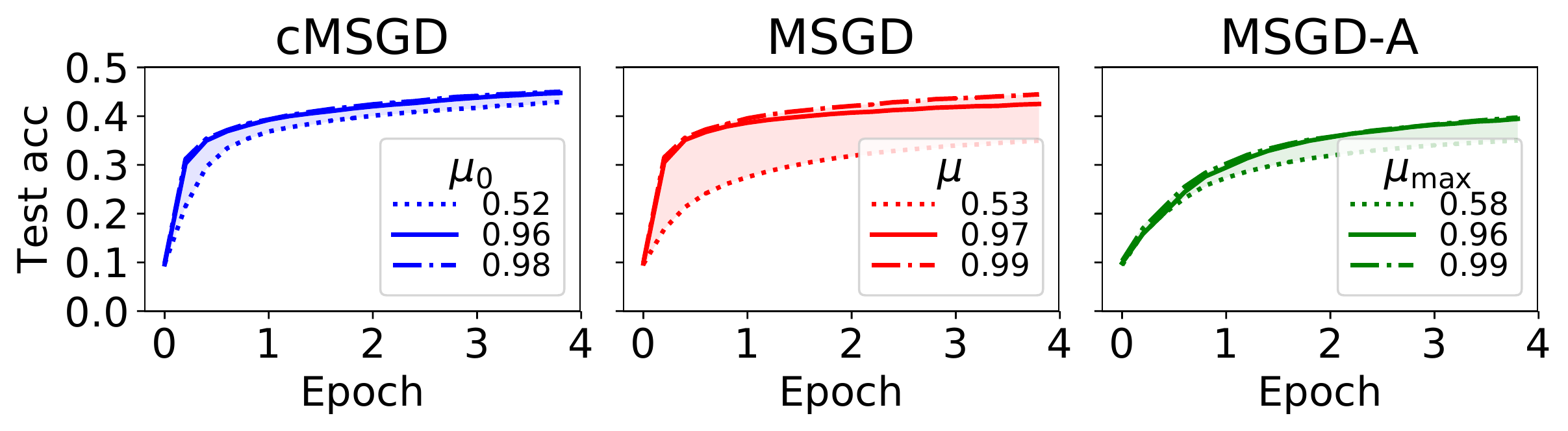}}

	\subfloat[C1 (CNN, CIFAR-10)]{\includegraphics[width=12cm]{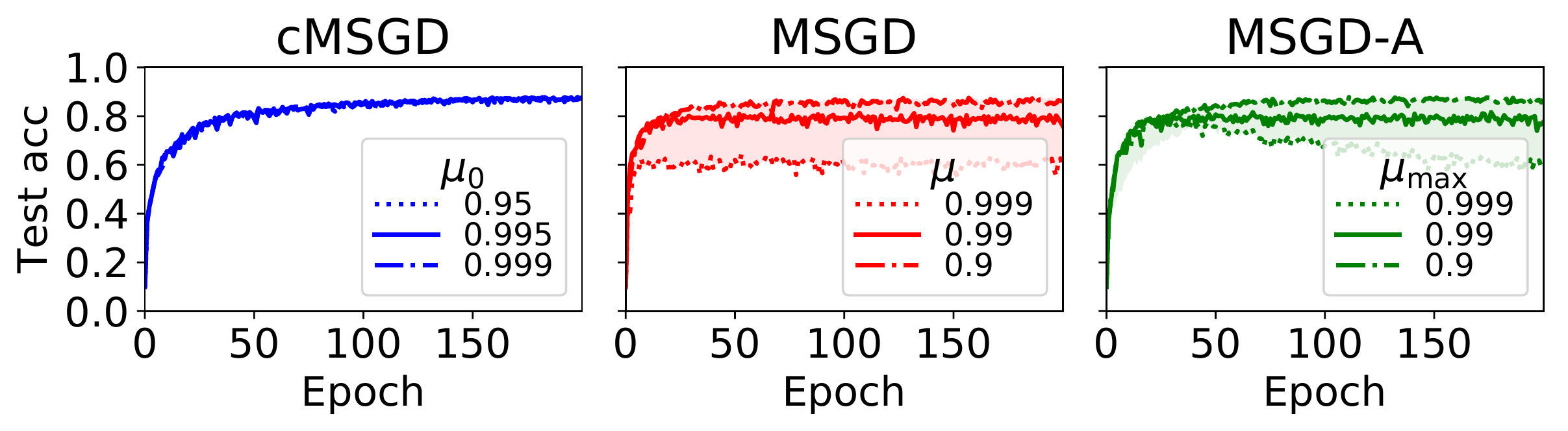}}
	\caption{cMSGD vs MSGD and MSGD-A on the same three models. We set $\eta=$1e-2 for M0 and $\eta=$1e-3 for C0, C1. For M0 and C0, we perform a log-uniform random search for $1-\mu_0$ and $1-\mu$ in [5e-3,5e-1]. For C1, we sample $\mu_0,\mu,\mu_\text{max}\in$\{0.9,0.95,0.99,0.995,0.999\}. The remaining set-up is identical to that in Fig.~\ref{fig:MNIST_test_lr}. Again, we observe that cMSGD is an adaptive scheme that is robust to varying hyper-parameters and network structures, and out-performs MSGD and MSGD-A. }
	\label{fig:compare_cmsgd}
\end{figure}
\begin{figure}[t!]
	\centering
	\subfloat[M0 (fully connected NN, MNIST)]{\includegraphics[width=12cm]{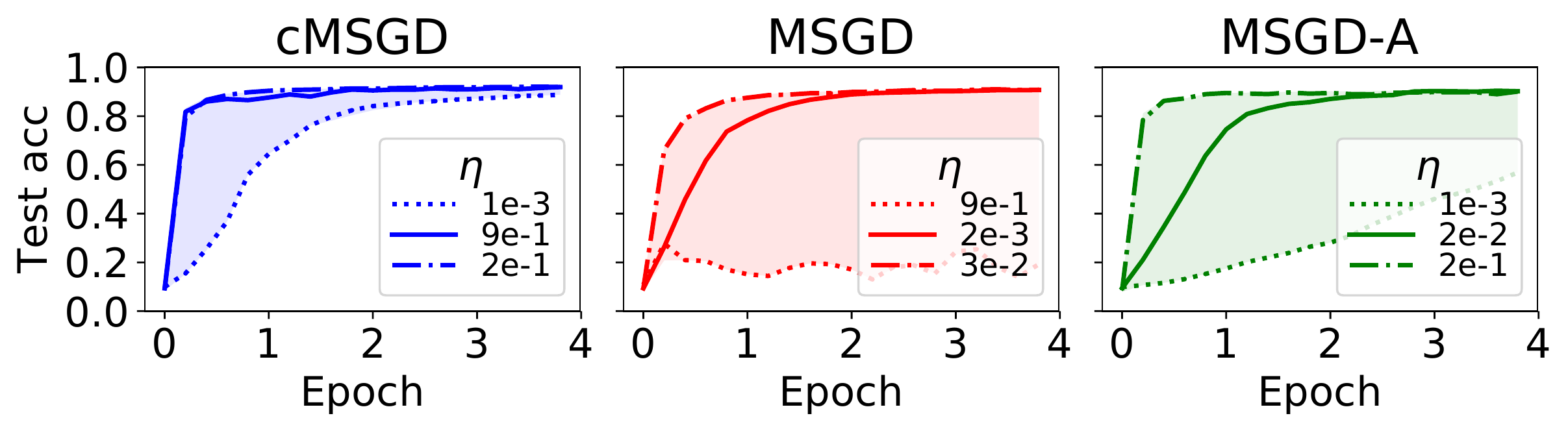}}
	
	\subfloat[C0 (fully connected NN, CIFAR-10)]{\includegraphics[width=12cm]{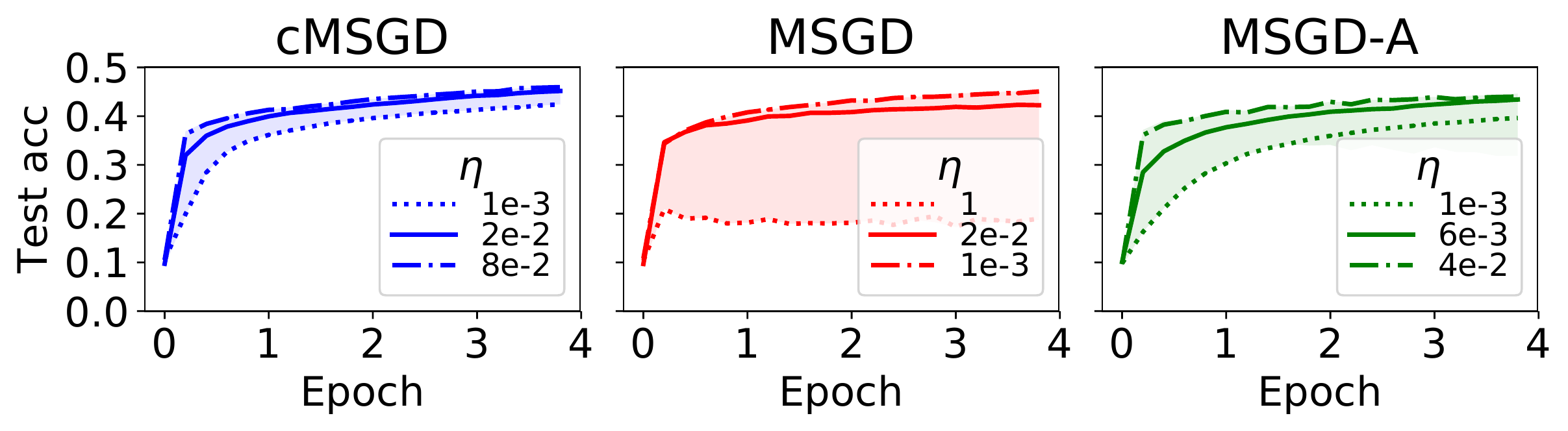}}
	
	\subfloat[C1 (CNN, CIFAR-10)]{\includegraphics[width=12cm]{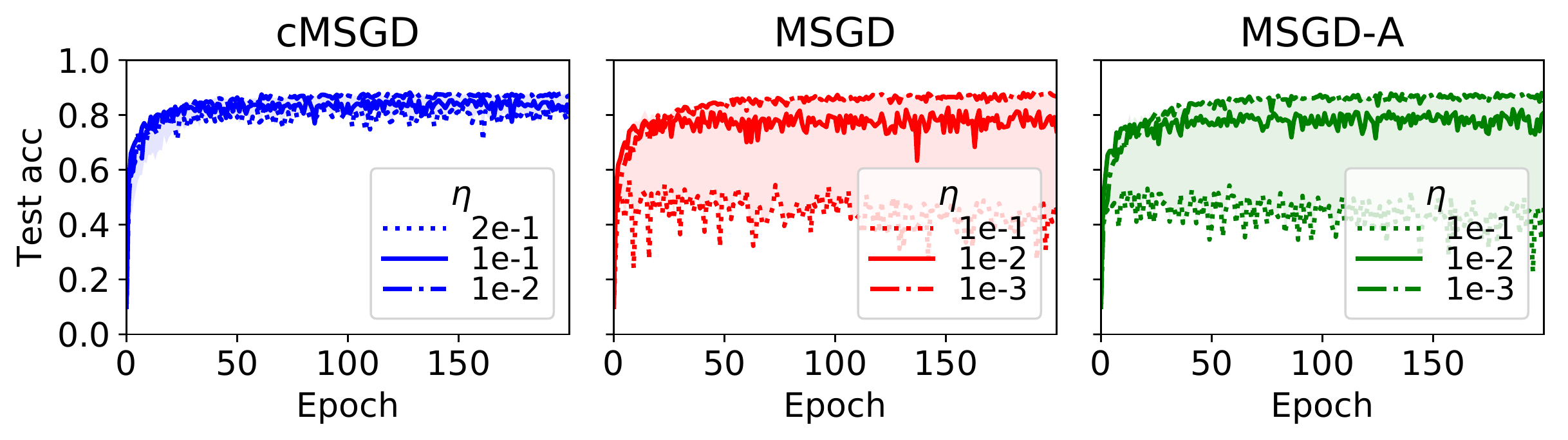}}	
	\caption{Comparing the sensitivity of cMSGD, MSGD and MSGD-A to different learning rates. The set-up is same as that in Fig.~\ref{fig:compare_cmsgd} except that for MSGD and MSGD-A, we now fix $\mu,\mu_\text{max}$ to be the best values found in Fig.~\ref{fig:compare_cmsgd} for each experiment, but we vary the learning rate in the ranges: M0 and C0: $\eta\in$[1e-3,1], C1: $\eta\in$\{1e-3,5e-3,1e-2,2e-2,1e-1\}. For cMSGD, we saw from Fig.~\ref{fig:compare_cmsgd} that the value of $\mu_0$ is mostly inconsequential, so we simply set $\mu_0=0$ and vary $\eta$ in the same ranges. We observe that the unlike MSGD and MSGD-A, cMSGD is generally robust to changing learning rates and this further confirms its adaptive properties.}
	\label{fig:mom_lr_sens}
\end{figure}
\begin{table}
	\caption{Best average test accuracy found by random/grid search.}
	\label{tab:test_acc}
	\vskip 0.15in
	\begin{center}
		\begin{tabular}{lccc||cccr}
			\hline
			& cSGD & Adagrad & Adam & cMSGD & MSGD & MSGD-A \\
			\hline
			M0    & \textbf{0.925}	& 0.923 &	0.924 & \textbf{0.924} & 0.914 & 0.908 \\
			C0    & \textbf{0.461}	& 0.457	& 0.460	& \textbf{0.461} & 0.453 & 0.446 \\					
			C1    & 0.875 & 0.878 & \textbf{0.881} & \textbf{0.876} & 0.868 & 0.869 \\
			\hline
		\end{tabular}
	\end{center}
\end{table}
\section{Related Work}
\label{sec:related_work}
Classical bound-type convergence results for SGD and variants include
\citet{moulines2011non,shamir2013stochastic,bach2013non,needell2014stochastic,xiao2014proximal,shalev2014accelerated}.
Our approach differs in that we obtain precise, albeit only distributional, descriptions of the SGD dynamics that hold in non-convex situations. 

In the vein of continuous approximation to stochastic algorithms, a related body of work is stochastic approximation theory~\citep{kushner2003stochastic,ljung2012stochastic}, which establish ODEs as almost sure limits of trajectories of stochastic algorithms. In contrast, we obtain SDEs that are weak limits that approximate not individual sample paths, but their distributions. Other deterministic continuous time approximation methods include~\citet{su2014differential,krichene2015accelerated,Wibisono22112016}.

Related work in SDE approximations of the SGD are~\citet{mandt2015continuous,mandt2016variational}, where the authors derived the first order SME heuristically. In contrast, we establish a rigorous statement for this type of approximations (Thm.~\ref{thm:sme}). Moreover, we use asymptotic analysis and control theory to translate understanding into practical algorithms. Outside of the machine learning literature, similar modified equation methods also appear in numerical analysis of SDEs~\citep{zygalakis2011existence} and quantifying uncertainties in ODEs~\citep{conrad2015probability}. 

The second half of our work deals with practical problems of adaptive selection of the learning rate and momentum parameter. There is abundant literature on learning rate adjustments, including annealing schedules~\cite{robbins1951stochastic,moulines2011non,xu2011towards,shamir2013stochastic}, adaptive per-element adjustments~\cite{duchi2011adaptive,zeiler2012adadelta,hinton2012lecture,kingma2015adam} and meta-learning~\cite{andrychowicz2016learning}. Our approach differs in that optimal control theory provides a natural, non-black-box framework for developing dynamic feed-back adjustments, allowing us to obtain adaptive algorithms that are truly robust to changing model settings. Our learning rate adjustment policy is similar to~\citet{schaul2013no,schaul2013adaptive} based on one-step optimization, although we arrive at it from control theory. Our method may also be easier to implement because it does not require estimating diagonal Hessians via back-propagation. 
There is less literature on momentum parameter selection. A heuristic annealing schedule (referred to as MSGD-A earlier) is suggested in~\citet{sutskever2013importance}, based on the original work of~\citet{nesterov1983method}. The choice of momentum parameter in deterministic problems is discussed in~\citet{qian1999momentum,nesterov2013introductory}. To the best of our knowledge, a systematic stochastic treatment of adaptive momentum parameter selection for MSGD has not be considered before. 

\section{Conclusion and Outlook}
\label{sec:conclusion}
Our main contribution is twofold. First, we propose the SME as a unified framework for quantifying the dynamics of SGD and its variants, beyond the classical convex regime. Tools from stochastic calculus and asymptotic analysis provide precise dynamical description of these algorithms, which help us understand important phenomena, such as descent-fluctuation transitions and the nature of acceleration schemes. Second, we use control theory as a natural framework to derive adaptive adjustment policies for the learning rate and momentum parameter. This translates to robust algorithms that requires little tuning across multiple datasets and model choices. 

An interesting direction of future work is extending the SME framework to develop adaptive adjustment schemes for other hyper-parameters in SGD variants, such as Polyak-Ruppert Averaging~\citep{polyak1992acceleration}, SVRG~\citep{johnson2013accelerating} and elastic averaging SGD~\citep{zhang2015deep}. More generally, the SME framework may be a promising methodology for the analysis and design of stochastic gradient algorithms and beyond. 



\newpage
\appendix
\onecolumn
\section{Modified equations in the numerical analysis of PDEs}
\label{app:sec:modified_eqn_PDE}
The method of modified equations is widely applied in finite difference methods in numerical solution of PDEs~\cite{hirt1968heuristic,noh1960difference,daly1963stability,warming1974modified}.  
In this section, we briefly demonstrate this classical method. Consider the one dimensional transport equation
\begin{equation}
\frac{\partial{u}}{\partial{t}} = c \frac{\partial{u}}{\partial{x}} 
\label{eq:transport}
\end{equation}
where $u:[0,T]\times[0,L]\rightarrow \mathbb{R}$ represents a density of some material in $[0,L]$ and $c>0$ is the transport velocity. It is well-known that the simple forward-time-central-space differencing leads to instability for all discretization step-sizes~\citep{leveque2002finite}. Instead, more sophisticated differencing schemes must be used. 

We set time and space discretization steps to $\Delta t$ and $\Delta x$ and denote $u(n\Delta t,j\Delta x) = U_{n,j}$ for $1\leq n \leq N$ and $1\leq j \leq J$. The simplest scheme that can exhibit stability is the {\it upwind scheme}~\citep{courant1952solution}, where we approximate~\eqref{eq:transport} by the difference equation
\begin{equation}
U_{n+1,j} = U_{n,j} + \Delta t \left(c^+ \frac{U_{n,j+1}-U_{n,j}}{\Delta x} + c^- \frac{U_{n,j}-U_{n,j-1}}{\Delta x}  \right),
\label{eq:upwind}
\end{equation}
where $c^+ = \max (c,0)$ + $c^- = \min (c,0)$. The idea is to now approximate this difference scheme by another continuous PDE, that is not equal to the original equation~\eqref{eq:transport} for non-zero $\Delta x,\Delta t$. This can be done by Taylor expanding each term in~\eqref{eq:upwind} around $u(t,x)=U_{n,j}$. Simplifying and truncating to leading term in $\Delta t,\Delta x$, we obtain the modified equation
\begin{equation}
\frac{\partial u}{\partial t} - c \frac{\partial u}{\partial x}= \frac{1}{2} c \Delta x (1-r) \frac{\partial^2 u}{\partial x^2},
\label{eq:upwind_me}
\end{equation}
where $r=c\Delta t/\Delta x$ is the Courant-Friedrichs-Lewy (CFL) number~\citep{courant1952solution}. Notice that in the limit $\Delta t, \Delta x \rightarrow 0$ with $r$ fixed, one recovers the original transport equation, but for finite step sizes, the upwind scheme is really described by the modified equation~\eqref{eq:upwind_me}. In other words, this truncated equation describes the leading order, non-trivial behavior of the finite difference scheme.

From the modified equation~\eqref{eq:upwind_me}, one can immediately deduce a number of interesting properties. First, the error from the upwind scheme is diffusive in nature, due to the presence of the second order spatial derivative on the right hand side. Second, we observe that if the CFL number $r$ is greater than 1, then the coefficient for the diffusive term becomes negative and this results in instability. This is the well-known {\it CFL condition}. This places a fundamental limit on the spatial resolution for fixed temporal resolution with regards to the stability of the algorithm. Lastly, the error term is proportional to $\Delta x$ for fixed $r$, thus it may be considered a first order method. 

Now, another possible proposal for discretizing~\eqref{eq:transport} is the Lax-Wendroff (LW) scheme~\citep{lax1960systems}: 
\begin{align}
U_{n+1,j} &= U_{n,j} + \Delta t \left(c \frac{U_{n,j+1}-U_{n,j-1}}{2\Delta x} - c\Delta t \frac{U_{n,j+1}-2U_{n,j}+U_{n,j-1}}{2\Delta x^2}\right),
\label{eq:lf}
\end{align}
whose modified equation is
\begin{equation}
\frac{\partial u}{\partial t} - c \frac{\partial u}{\partial x} = \frac{1}{6} c \Delta x^2 (r^2-1) \frac{\partial^3 u}{\partial x^3}.
\label{eq:lf_me}
\end{equation}
Comparing with~\eqref{eq:upwind_me}, we observe that the LW scheme error is of higher order ($\Delta x^2$), but at the cost of introducing dispersive, instead of diffusive errors due to the presence of the third derivative. These findings are in excellent agreement with the actual behavior of their respective discrete numerical schemes~\citep{warming1974modified}. 

We stress here that if we simply took the trivial leading order, the right hand sides of 
~\eqref{eq:upwind_me} and \eqref{eq:lf_me} disappear and vital information, including stability, accuracy and the nature of the error term will be lost. The ability to capture the effective dynamical behavior of finite difference schemes is the key strength of the modified equations approach, which has become the primary tool in analyzing and improving finite difference algorithms. The goal of our work is to extend this approach to analyze stochastic algorithms. 

\section{Summary of SDE terminologies and results}
\label{app:sec:SDE_stoch_calculus}
Here, we summarize various SDE terminologies and results we have used throughout the main paper and also subsequent derivations. A particular important result is the It\^{o} formula (Sec.~\ref{app:sec:ito_formula}), which is used throughout this work for deriving moment equations. For a thorough reference on the subject of stochastic calculus and SDEs, we suggest~\cite{oksendal2013stochastic}.

\subsection{Stochastic differential equations}
Let $T>0$. An It\^{o} stochastic differential equation on the interval $[0,T]$ is an equation of the form
\begin{equation}
dX_t = b(X_t,t)dt + \sigma(X_t,t)dW_t, \qquad X_0 = x_0,
\label{app:eq:sde_general}
\end{equation}
where $X_t\in \mathbb{R}^d$, $b:\mathbb{R}^d\times [0,T]\rightarrow \mathbb{R}^d$, $\sigma:\mathbb{R}^d\times [0,T]\rightarrow \mathbb{R}^{d\times l}$ and $W_t$ is a $l$-dimensional Wiener process, or Brownian motion. This is a mathematically sound way of expressing the intuitive notion of SDEs being ODEs plus noise:
\begin{equation}
\dot{X}_t = b(X_t,t) + \text{``noise''}
\end{equation}
The equation~\eqref{app:eq:sde_general} is really a ``shorthand'' for the integral equation
\begin{equation}
X_t - x_0 = \int_{0}^{t} b(X_s,s) ds + \int_{0}^{t} \sigma(X_s,s)dW_s.
\label{app:eq:sde_general_integral}
\end{equation}
The last integral is defined in the It\^{o} sense, i.e.
\begin{equation}
\int_{0}^{t} F_s dW_s := \lim_{n\rightarrow\infty} \sum_{[s_{i-1},s_{i}]\in\pi_n} F_{s_{i-1}}(W_{s_{i}}-W_{s_{i-1}}),
\end{equation}
where $\pi_n$ is a sequence of $n$-partitions of $[0,t]$ and the limit represents convergence in probability. 
In~\eqref{app:eq:sde_general}, $b$ is known as the {\it drift}, and $\sigma$ is known as the {\it diffusion matrix}. When they satisfy Lipschitz conditions, one can show that~\eqref{app:eq:sde_general} (or~\eqref{app:eq:sde_general_integral}) has a unique strong solution (\cite{oksendal2013stochastic}, Chapter 5). For our purposes in this paper, we consider the special case where $b,\sigma$ do not depend on time and we set $d=l$ so that $\sigma$ is a square matrix.

To perform calculus, we need an important result that generalizes the notion of chain rule to the stochastic setting. 

\subsection{It\^{o} formula}
\label{app:sec:ito_formula}
It\^{o} formula, also known as It\^{o}'s lemma, is the extension of the chain rule of ordinary calculus to the stochastic setting. Let $\phi \in C^{2,1}(\mathbb{R}^d\times[0,T])$ and let $X_t$ be a stochastic process satisfying the SDE~\eqref{app:eq:sde_general}, and thus~\eqref{app:eq:sde_general_integral}. Then, the stochastic process $\phi(X_t,t)$ is again an It\^{o} process satisfying
\begin{align}
d\phi(X_t,t) =&  \left[ \partial_t \phi(X_t,t) + \left(\nabla \phi(X_t,t)^T b(X_t,t) + \frac{1}{2}\text{Tr}[\sigma(t,X_t)^T H\phi(t,X_t) \sigma(t,X_t)] \right)\right]dt \nonumber\\
& + \left[\nabla \phi(X_t,t)^T \sigma(X_t,t)\right] dW_t,
\label{app:eq:ito_formula}
\end{align}
where $\nabla$ denotes gradient with respect to the first argument and $H\phi$ denotes the Hessian, i.e. $H\phi_{(ij)} = \partial^2 \phi / \partial x_{(i)}\partial x_{(j)}$. The formula~\eqref{app:eq:ito_formula} is the It\^{o} formula. If $\phi$ is not a scalar but a vector, then each of its component satisfy~\eqref{app:eq:ito_formula}. Note that if $\sigma=0$, this reduces to the chain rule of ordinary calculus.

\subsection{The Ornstein-Uhlenbeck process}
\label{app:sec:ou}
An important solvable SDE is the Ornstein-Uhlenbeck (OU) process~\cite{uhlenbeck1930theory}. Consider $d=1$, $b(x,t)= \theta (\xi-x)$ and $\sigma(x,t)=\sigma>0$, with $\theta>0$, $\sigma>0$ and $\xi\in\mathbb{R}$. Then we have the SDE
\begin{equation}
dX_t = \theta(\xi - X_t) dt + \sigma dW_t, \qquad X_0 = x_0.
\label{app:eq:ou_sde}
\end{equation}
To solve this equation, we change variables $x\mapsto\phi(x,t)=x e^{\theta t}$. Applying It\^{o} formula, we have
\begin{equation}
d\phi(X_t,t) = \theta \xi e^{\theta t} dt + \sigma e^{\theta t} dW_t,
\end{equation}
which we can integrate from $0$ to $T$ to get
\begin{equation}
X_t = x_0 e^{-\theta t}  + \xi (1-e^{-\theta t}) + \sigma \int_{0}^{t} e^{-\theta(t-s)} dW_s.
\label{app:eq:ou_pathwise_soln}
\end{equation}
This is a path-wise solution to the SDE~\eqref{app:eq:ou_sde}. To infer distributional properties, we do not require such precise solutions. In fact, we only need the distribution of the random variable $X_t$ at any fixed time $t\in[0,T]$. Observe that $X_t$ is really a Gaussian process, since the integrand in the Wiener integral is deterministic. Hence, we need only calculate its moments. Taking expectation on~\eqref{app:eq:ou_pathwise_soln}, we get 
\begin{equation}
\E X_t = x_0 e^{-\theta t} + \xi (1-e^{-\theta t}).
\end{equation}
To obtain the covariance function, we see that 
\begin{equation}
\E (X_t-\E X_t) (X_s-\E X_s) = \sigma^2 \E \left[\int_{0}^{t}e^{\theta(u-s)}dW_u\int_{0}^{t}e^{\theta(v-s)}dW_v\right].
\end{equation}
This can be evaluated by using {\it It\^{o}'s isometry}, which says that for any $W_t$ adapted process $\phi_t,\psi_t$, we have
\begin{equation}
\E \left[\int_{0}^{t}\phi_u dW_u   \int_{0}^{t}\psi_v dW_v  \right] = \E \left[\int_{0}^{t} \phi_s \psi_s ds\right]. 
\end{equation}
We get, for $s\leq t$
\begin{equation}
\text{cov}(X_s,X_t) = \frac{\sigma^2}{2\theta}\left(e^{-\theta\vert t-s\vert} + e^{-\theta\vert t+s\vert}\right),
\end{equation}
and in particular, for fixed $t\in [0,T]$, we have
\begin{equation}
\text{Var} (X_t) = \frac{\sigma^2}{2\theta}( 1 - e^{-2\theta t} ). 
\end{equation}
Hence, we have
\begin{equation}
X_t \sim \mathcal{N} \left(x_0 e^{-\theta t} + \xi (1-e^{-\theta t}),  \frac{\sigma^2}{2\theta}( 1 - e^{-2\theta t} )\right).
\end{equation}
In Sec. 3.1 in the main paper, the solution of the SME is the OU process with $\theta=2(1+\eta),\xi=0,\sigma=2\sqrt{\eta}$. Making these substitutions, we obtain
\begin{equation}
X_t \sim \mathcal{N} \left(x_0 e^{-2(1+\eta)t},\frac{\eta}{1+\eta} \left(1-e^{-4(1+\eta)t}\right)\right).\nonumber
\label{eq:app:ou_dist}
\end{equation}

\subsection{Numerical solution of SDEs}
\label{app:sec:numerical}
Unfortunately, most SDEs are not amenable to exact solutions. Often, we resort to numerical methods. The simplest method is the {\it Euler-Maruyama method}. This extends the Euler method for ODEs to SDEs. Fix a time discretization size $\delta>0$ and define $\tilde{X}_k = X_{k\delta}$, then we can iterate the finite difference equation
\begin{equation}
\tilde{X}_{k+1} = \tilde{X}_{k} + \delta b(\tilde{X}_k,k\delta) + \sigma(\tilde{X}_k,k\delta) (W_{(k+1)\delta}-W_{k\delta}).
\end{equation}
By definition, $W_{(k+1)\delta}-W_{k\delta} \sim \mathcal{N}(0,\delta I)$, and are independent for each $k$. Here, $I$ is the identity matrix. Hence, we have the Euler-Maruyama scheme
\begin{equation}
\tilde{X}_{k+1} = \tilde{X}_{k} + \delta b(\tilde{X}_k,k\delta) + \sqrt{\delta}\sigma(\tilde{X}_k,k\delta)Z_k,
\label{app:eq:em_method}
\end{equation}
where $Z_k \overset{i.i.d.}{\sim} \mathcal{N}(0,I)$.

One can show that the Euler-Maruyama method~\eqref{app:eq:em_method} is a first order weak approximation (c.f. Def. 1 in main paper) to the SDE~\eqref{app:eq:sde_general}. However, it is only a order $1/2$ scheme in the strong sense~\citep{Kloeden1992}, i.e. 
\begin{equation}
\mathbb{E}\vert X_{k\delta} - \tilde{X}_{k\delta}\vert < C \delta^\frac{1}{2}.
\end{equation}
With more sophisticated methods, one can design higher order schemes (both in the strong and weak sense), see~\cite{mil1986weak}. 

\subsection{Stochastic asymptotic expansion}
Besides numerics, if there exists small parameters in the SDE, we can proceed with stochastic asymptotic expansions~\cite{freidlin2012random}. This is the case for the SME, which has a small $\eta^{1/2}$ multiplied to the noise term. Let us consider a time-homogeneous SDE of the form
\begin{equation}
dX^{\epsilon}_t = b(X^{\epsilon}_t)dt + \epsilon \sigma(X^{\epsilon}_t)dW_t
\label{app:eq:asymp_sde}
\end{equation}
where $\epsilon\ll 1$. The idea is to follow standard asymptotic analysis and write $X^{\epsilon}_t$ as an asymptotic series
\begin{equation}
X^{\epsilon}_t = X_{0,t} + \epsilon X_{1,t} + \epsilon^2 X_{2,t} + \dots.
\label{app:eq:asymp_exp}
\end{equation}
We substitute~\eqref{app:eq:asymp_exp} into~\eqref{app:eq:asymp_sde} and assuming smoothness of $b$ and $\sigma$, we expand
\begin{align}
b_{\epsilon}(X^{\epsilon}_t) &= b(X_{0,t}) + \epsilon \nabla b(X_{0,t})X_{1,t} + \mathcal{O}(\epsilon^2) \nonumber\\
\sigma(X^{\epsilon}_t) &= \sigma(X_{0,t}) + \epsilon \nabla \sigma(X_{0,t})X_{1,t} + \mathcal{O}(\epsilon^2)
\end{align}
to get 
\begin{align}
dX_{0,t} &= b(X_{0,t}) dt, \nonumber \\
dX_{1,t} &= \nabla b(X_{0,t})X_{1,t} dt + \sigma(X_{0,t}) dW_t, \nonumber \\
&\vdots
\label{app:eq:asymp_2_gen}
\end{align}
and $X_{0,0}=x_0,X_{1,0}=0$. In general, the equation for $X_{i,t}$ are linear stochastic differential equations with time-dependent coefficients depending on $\{X_{0,t},X_{1,t},\dots,X_{i-1,t}\}$ and the initial conditions are $X_{0,0}=x_0$, $X_{i,0}=0$ for all $i\geq 1$. Hence, the asymptotic equations can be solved sequentially to obtain an estimate of $X_t$ to arbitrary order in $\epsilon$. The equations for higher order terms become messy quickly, but they are always linear in the unknown, as long as all the previous equations are solved. For more details on stochastic asymptotic expansions, the reader is referred to~\cite{freidlin2012random}. 

\subsection{Asymptotics of the SME}
We now derive the first two asymptotic equations of the SME. we take $\epsilon=\sqrt{\eta}$, $b=-\nabla f$ ($\mathcal{O}(\eta)$ term can be ignored for first two terms) and $\sigma=\Sigma^{1/2}$. Then,~\eqref{app:eq:asymp_2_gen} becomes 
\begin{align}
dX_{0,t} &= -\nabla f(X_{0,t}) dt,  \label{eq:asymp_0}\\
dX_{1,t} &= -Hf(X_{0,t})X_{1,t} dt + \Sigma(X_{0,t})^\frac{1}{2}dW_t, \label{eq:asymp_1}
\end{align}
where $Hf_{(ij)}=\partial_{(i)}\partial_{(j)}f$ is the Hessian of $f$. 

In the following analysis, we shall assume that the truncated series approximation
\begin{equation}
\hat{X}_t = X_{0,t}+\sqrt{\eta}X_{1,t},
\end{equation}
where $X_{0,t},X_{1,t}$ satisfy~\eqref{eq:asymp_0} and~\eqref{eq:asymp_1}, describes the leading order stochastic dynamics of the SGD. Now, let us analyze the asymptotic equations in detail. First, we assume that the ODE~\eqref{eq:asymp_0} has a unique solution $X_{0,t},t\geq 0$ with $X_{0,0}=x_0$. This is true if for example, $\nabla f$ is locally Lipschitz. Next, let us define the non-random functions
\begin{align}
H_t &= Hf(X_{0,t}),\nonumber\\
\Sigma_t &= \Sigma(X_{0,t}).
\label{eq:nonrandom}
\end{align}
Both $H$ and $\sigma$ are $d\times d$ matrices for each $t$. Then,~\eqref{eq:asymp_1} becomes the time-inhomogeneous linear SDE
\begin{equation}
dX_{1,t} = -H_t X_{1,t} + \Sigma_t^{\frac{1}{2}} dW_t,
\end{equation}
with $X_{1,0}=0$. Since the drift is linear and the diffusion matrix is constant (i.e. independent of $X_{1,t}$), $X_{1,t}$ is a Gaussian process. Hence we need only calculate its mean and covariance using It\^{o} formula (see~\ref{app:sec:ito_formula}). We have
\begin{equation}
\E X_{1,t} = 0,
\end{equation}
and the covariance matrix $S_t=\text{Cov}(X_{1,t})$ satisfies the differential equation
\begin{equation}
\frac{d}{dt} S_t = -S_tH_t-H_tS_t + \Sigma_t,
\label{eq:general_cov}
\end{equation}
with $S_0=0$. This equation is a linearized version of the {\it Riccati equation} and there are simple closed-form solutions under special conditions, e.g. $d=1$ or $H_t$ is constant.

Hence, we conclude that the asymptotic approximation $\hat{X}_t$ is a Gaussian process with distribution
\begin{equation}
\hat{X}_t \sim \mathcal{N} (X_{0,t},\eta S_t),
\label{eq:general_distr}
\end{equation}
where $X_{0,t}$ solves the ODE~\eqref{eq:asymp_0} and $S_t$ solves the ODE~\eqref{eq:general_cov}, with $H_t,\Sigma_t$ given by~\eqref{eq:nonrandom}. 

\begin{remark}
	At this point, it is important to discuss the validity of the asymptotic approximation~\eqref{eq:general_distr}, and the SME approximation~\eqref{eq:sme} in general. What we prove in Sec.~\ref{app:sec:proof} and is shown in~\cite{freidlin2012random} is that for fixed $T$, we can take $\eta=\eta(T)$ small enough so that the SME and its asymptotic expansion is a good approximation of the distribution of the SGD iterates. What we did not prove is that for fixed $\eta$, the approximations hold for arbitrary $T$. In particular, it is not hard to construct systems where for fixed $\eta$, both the SME and the asymptotic expansion fails when $T$ is large enough. To prove the second general statement requires further assumptions, particularly on the distribution of $f_i$'s. This is out of the scope of the current work. 
\end{remark}

\section{Formal Statement and proof of Thm. 1}
\label{app:sec:proof}
\begin{theorem} [Stochastic modified equations]
	\label{thm:sme}
	Let $\alpha\in\{1,2\}$, $0<\eta<1$, $T>0$ and set $N=\lfloor T/\eta\rfloor$. Let $x_k\in\mathbb{R^d}$, $0\leq k\leq N$ denote a sequence of SGD iterations defined by (2). Define $X_t\in \mathbb{R}^d$ as the stochastic process satisfying the SDE
	\begin{equation}
	dX_{t} = -\nabla (f(X_t) + \frac{1}{4}(\alpha-1)\eta \vert \nabla f(X_t) \vert^2)dt + (\eta\Sigma(X_t))^{\frac{1}{2}}dW_t
	\label{eq:sme}
	\end{equation}
	$X_0=x_0$ and $\Sigma(x) = \frac{1}{n} \sum_{i=1}^n (\nabla f(x) - \nabla f_i(x))(\nabla f(x) - \nabla f_i(x))^T$.
	
	Fix some test function $g\in G$ (c.f. Def. 1 in main paper). Suppose further that the following conditions are met:
	
	\begin{enumerate}[(i)]
		\item $\nabla f,\nabla f_i$ satisfy a Lipschitz condition: there exists $L>0$ such that
		\begin{equation}
		\vert \nabla f(x)-\nabla f(y) \vert + \sum_{i=1}^{n} \vert \nabla f_i(x)-\nabla f_i(y) \vert   \leq L \vert x-y \vert .\nonumber
		\end{equation}
		\item $f,f_i$ and its partial derivatives up to order $7$ belong to $G$.
		\item $\nabla f,\nabla f_i$ satisfy a growth condition: there exists $M>0$ such that 
		\begin{equation}
		\vert \nabla f(x) \vert + \sum_{i=1}^{n} \vert \nabla f_i(x) \vert   \leq M (1+\vert x \vert) .\nonumber	
		\end{equation}
		\item $g$ and its partial derivatives up to order $6$ belong to $G$.
	\end{enumerate} 
	Then, there exists a constant $C>0$ independent of $\eta$ such that for all $k=0,1,\dots,N$, we have 
	\begin{equation}
	\vert \mathbb{E}g(X_{k\eta}) - \mathbb{E}g(x_k) \vert \leq C \eta^\alpha.\nonumber
	\end{equation}
	That is, the equation~\eqref{eq:sme} is an order $\alpha$ weak approximation of the SGD iterations. 
\end{theorem}
The basic idea of the proof is similar to the classical approach in proving weak convergence of discretization schemes of SDEs outlined in the seminal papers by Milstein (\cite{mil1975approximate,mil1979method,mil1986weak,milstein1995numerical}). The main difference is that we wish to establish that the continuous SME is an approximation of the discrete SGD, instead of the other way round, which is the case dealt by classical approximation theorems of SDEs with finite difference schemes. In the following, we first show that a one-step approximation has order $\eta^{\alpha+1}$ error, and then deduce, using the general result in~\cite{mil1986weak}, that the overall global error is of order $\eta^\alpha$. 

It is well known that a second order weak convergence discretization scheme for a SDE is not trivial. The classical Euler-Maruyama scheme, as well as the Milstein scheme are both first order weak approximations. However, in our case the problem simplifies significantly. This is because the noise we are trying to model is small, so that from the outset, we may assume that $b(x)=\mathcal{O}(1)$ but $\sigma(x)=\mathcal{O}(\eta^{1/2})$, i.e. we set $\sigma(x)=\eta^{1/2}\tilde{\sigma}(x)$ where $\tilde{\sigma}=\mathcal{O}(1)$ and deduce the appropriate expansions. For brevity, in the following we will drop the tilde and simply denote the noise term of the SDE by $\eta^{1/2}\sigma$. 

In the subsequent proofs we will make repeated use of Taylor expansions in powers of $\eta$. To simplify presentation, we introduce the shorthand that whenever we write $\mathcal{O}(\eta^\alpha)$, we mean that there exists a function $K(x)\in G$ (c.f. Def.~1 in main text) such that the error terms are bounded by $K(x)\eta^\alpha$. For example, we write
\begin{equation}
b(x+\eta) = b_0(x)+\eta b_1(x) + \mathcal{O}(\eta^2) 
\end{equation}
to mean: there exists $K \in G$ such that
\begin{equation}
\vert b(x+\eta) - b_0(x) - \eta b_1(x) \vert \leq K(x) \eta^2. 
\end{equation}
These results can be deduced easily using Taylor's theorem with a variety of forms of the remainder, e.g. Lagrange form. We omit such routine calculations. We also denote the partial derivative with respect to $x_{(i)}$ by $\partial_{(i)}$. 

First, let us prove a lemma regarding moments of SDEs with small noise.  
\begin{lemma}
	\label{app:lem:d}
	Let $0<\eta<1$. Consider a stochastic process $X_t$, $t\geq 0$ satisfying the SDE
	\begin{equation}
	dX_t = b(X_t)+\eta^{\frac{1}{2}}\sigma(X_t)dW_t
	\label{app:eq:sde}
	\end{equation}
	with $X_0 = x\in\mathbb{R}^d$ and $b,\sigma$ together with their derivatives belong to $G$. Define the one-step difference $\Delta = X_{\eta} - x$, then we have
	\begin{enumerate}[(i)]
		\item $\E \Delta_{(i)} = b_{(i)}\eta + \frac{1}{2}
		[\sum_{j=1}^{d}b_{(j)}\partial_{(j)}b_{(i)}]\eta^2 + \mathcal{O}(\eta^3)$.
		\item $\E \Delta_{(i)}\Delta_{(j)} = [b_{(i)}b_{(j)}+\sigma\sigma^T_{(ij)}]\eta^2 + \mathcal{O}(\eta^3)$.
		\item $\E \prod_{j=1}^s \Delta_{(i_j)} = \mathcal{O}(\eta^3)$ for all $s\geq 3$, $i_j = 1,\dots,d$.
	\end{enumerate}
	All functions above are evaluated at $x$. 
\end{lemma}
\begin{proof}
	One way to establish (i)-(iii) is to employ the Ito-Taylor expansion (see~\cite{Kloeden1992}, Chapter 5) on the random variable $X_\eta$ around $x$ and calculating the moments. Here, we will employ instead the method of semigroup expansions (see~\cite{hille1996functional}, Chapter XI), which works directly on expectation functions. The generator of the stochastic process~\eqref{app:eq:sde} is the operator $L$ acting on sufficiently smooth functions $\phi:\mathbb{R}^d\rightarrow\mathbb{R}$, and is defined by
	\begin{equation}
	L\phi = \sum_{i=1}^{d} b_{(i)}\partial_{(i)}\phi + \frac{1}{2}\eta^2\sum_{i,j=1}^{d}\sigma\sigma^T_{(ij)}
	\partial_{(i)}\partial_{(j)}\phi,
	\end{equation}
	A classical result on semigroup expansions (\cite{hille1996functional}, Chapter XI) states that if $\phi$ and its derivatives up to order 6 belong to $G$, then
	\begin{equation}
	\E\phi(X_\eta) = \phi(x)+L\phi(x)\eta + \frac{1}{2}L^2\phi(x)\eta^2+\mathcal{O}(\eta^3). 
	\label{app:eq:semigroup}
	\end{equation}
	Now, let $t\in\mathbb{R}^d$ and consider the moment-generating function (MGF)
	\begin{equation}
	M(t) = \E e^{t\cdot \Delta}.
	\label{app:eq:mgf}
	\end{equation}
	To ensure its existence we may instead set $t$ to be purely imaginary, i.e. $t=is$ where $s$ is real. Then,~\eqref{app:eq:mgf} is known as the characteristic function (CF). The important property we make use of is that the moments of $\Delta$ are found by differentiating the MGF (or CF) with respect to t. In fact, we have
	\begin{equation}
	\E \prod_{j=1}^{s} \Delta_{(i_j)} = \frac{\partial^s M(t)}{\prod_{j=1}^{s}\partial {t_{(i_j)}}}\biggr\vert_{t=0},
	\label{app:eq:mgf_moments}
	\end{equation}
	where $i_j=1,\dots,d$. We now expand $M(t)$ in powers of $\eta$ using formula~\eqref{app:eq:semigroup}. We get,
	\begin{align}
	M(t) =& 1 + \left[\sum_{i=1}^{d}b_{(i)}t_{(i)}\eta + \frac{1}{2}\sum_{i,j=1}^{d}b_{(i)}t_{(j)}\partial_{(i)}b_{(j)}\eta^2\right] \nonumber\\
	&+ \left[\frac{1}{2}\eta^2(\sum_{i=1}^{d}b_{(i)}t_{(i)})^2 + \frac{1}{2}\sum_{i,j=1}^{d} \sigma\sigma^T_{(ij)}t_{(i)}t_{(j)}\right] + \mathcal{O}(\eta^3). 
	\end{align}
	All functions are again evaluated at $x$. Finally, we apply formula~\eqref{app:eq:mgf_moments} to deduce (i)-(iii). 
\end{proof}

Next, we have an equivalent result for one SGD iteration. 
\begin{lemma}
	\label{app:lem:dbar}
	Let $0<\eta<1$. Consider $x_k$, $k\geq 0$ satisfying the SGD iterations
	\begin{equation}
	x_{k+1} = x_k - \eta \nabla f_{\gamma_k}(x_k)
	\label{app:eq:sgd}
	\end{equation}
	with $x_0 = x\in\mathbb{R}^d$. Define the one-step difference $\bar{\Delta} = x_{1} - x$, then we have
	\begin{enumerate}[(i)]
		\item $\E \bar{\Delta}_{(i)} = -\partial_{(i)} f\eta $
		\item $\E \bar{\Delta}_{(i)}\bar{\Delta}_{(j)} = \partial_{(i)} f\partial_{(j)} f\eta^2 
		+ \Sigma_{(ij)}\eta^2 $.
		\item $\E \prod_{j=1}^s \bar{\Delta}a_{(i_j)} = \mathcal{O}(\eta^3)$ for all $s\geq 3$, $i_j = 1,\dots,d$.
	\end{enumerate}
	where $\Sigma=\frac{1}{n} \sum_{i=1}^n (\nabla f - \nabla f_i)(\nabla f - \nabla f_i)^T.$ All functions above are evaluated at $x$. 
\end{lemma}
\begin{proof}
	From definition~\eqref{app:eq:sgd} and the definition of $\Sigma$, the results are immediate. 
\end{proof}

Now, we will need a key result linking one step approximations to global approximations due to Milstein. We reproduce the theorem, tailored to our problem, below. The more general statement can be found in~\cite{mil1986weak}.
\begin{theorem}[Milstein, 1986]
	\label{app:thm:milstein}
	Let $\alpha$ be a positive integer and let the assumptions in Theorem~\ref{thm:sme} hold. If in addition there exists $K_1,K_2\in G$ so that
	\begin{equation}
	\vert \E \prod_{j=1}^{s} \Delta_{(i_j)} - \E \prod_{j=1}^s \bar{\Delta}_{(i_j)} \vert \leq K_1(x)\eta^{\alpha+1},\nonumber
	\end{equation}
	for $s=1,2,\dots,2\alpha+1$ and 
	\begin{equation}
	\E \prod_{j=1}^{2\alpha+2} \vert \bar{\Delta}_{(i_j)} \vert \leq K_2(x)\eta^{\alpha+1}.\nonumber
	\end{equation}	
	Then, there exists a constant $C$ so that for all $k=0,1,\dots,N$ we have
	\begin{equation}
	\vert \E g(X_{k\eta}) - \E g(x_k)  \vert \leq C\eta^\alpha \nonumber
	\end{equation}
\end{theorem}
\begin{proof}
	See~\cite{mil1986weak}, Theorem 2 and Lemma 5. 
\end{proof}

\subsection*{Proof of Theorem~\ref{thm:sme}}
We are now ready to prove theorem~\ref{thm:sme} by checking the conditions in theorem~\ref{app:thm:milstein} with $\alpha=1,2$. The second condition is implied by Lemma~\ref{app:lem:dbar}. The first condition is implied by Lemma~\ref{app:lem:d} and Lemma~\ref{app:lem:dbar} with the choice
\begin{align}
b(x) &= -\nabla (f(x) + \frac{1}{4}\eta (\alpha-1) \vert \nabla f(x)) \vert^2, \nonumber \\
\sigma(x) &= \Sigma(x)^\frac{1}{2}.\nonumber 
\end{align}\qed

To illustrate our approximation result, let us calculate, using Monte-Carlo simulations, the weak error of the SME approximation
\begin{equation}
E_w = \vert \E g(X_{N\eta})-\E g(x_{N}) \vert,
\label{eq:weak_err}
\end{equation}
for $\alpha=1,2$ v.s. $\eta$ for different $f,f_i$ and generic polynomial test functions $g$. The results are shown in Fig.~\ref{fig:error_smg_sgd}. We see that we have order $\alpha$ weak convergence, even when some conditions of the above theorem are not satisfied (Fig.~\ref{fig:error_smg_sgd}(b)). 
\begin{figure}[H]
	\centering
	\subfloat[]{\includegraphics[width=6cm]{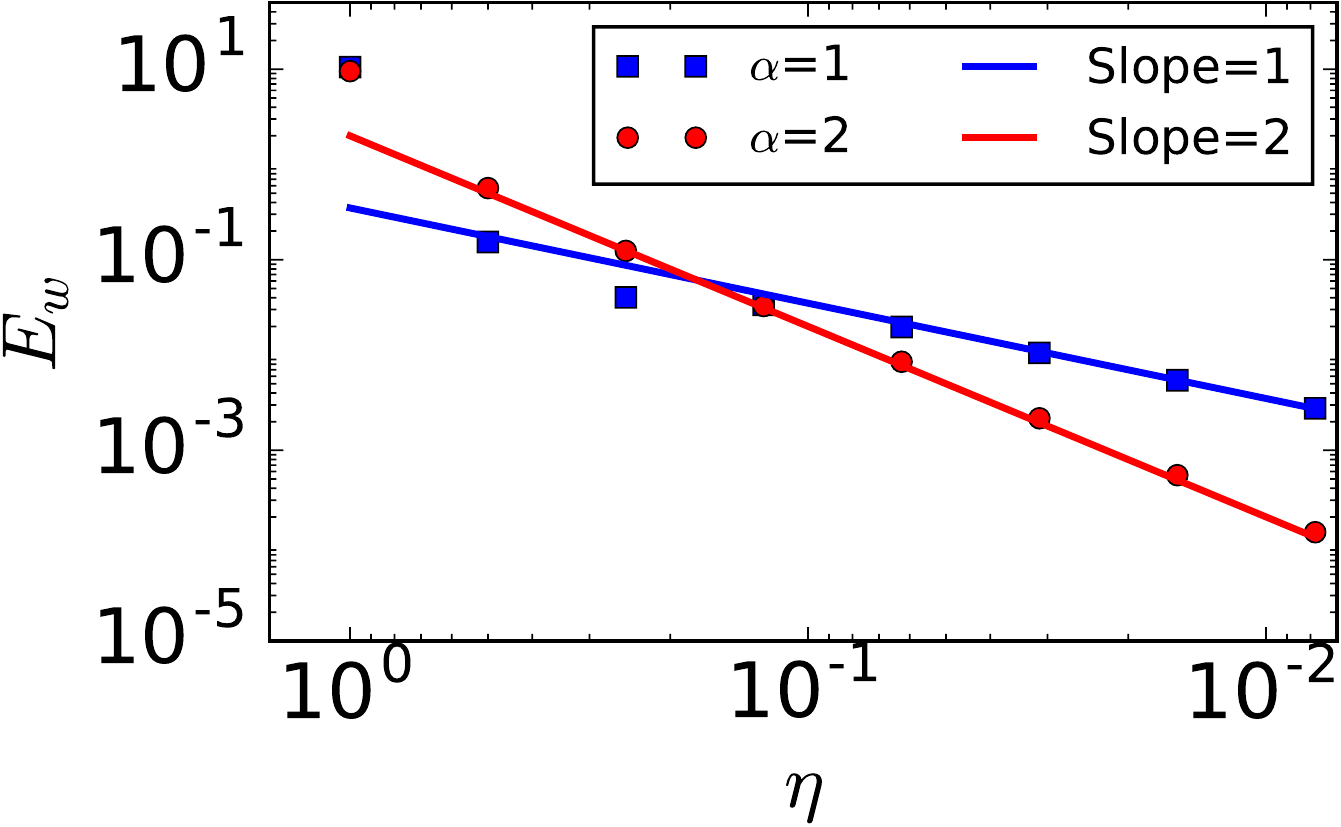}}
	\subfloat[]{\includegraphics[width=6cm]{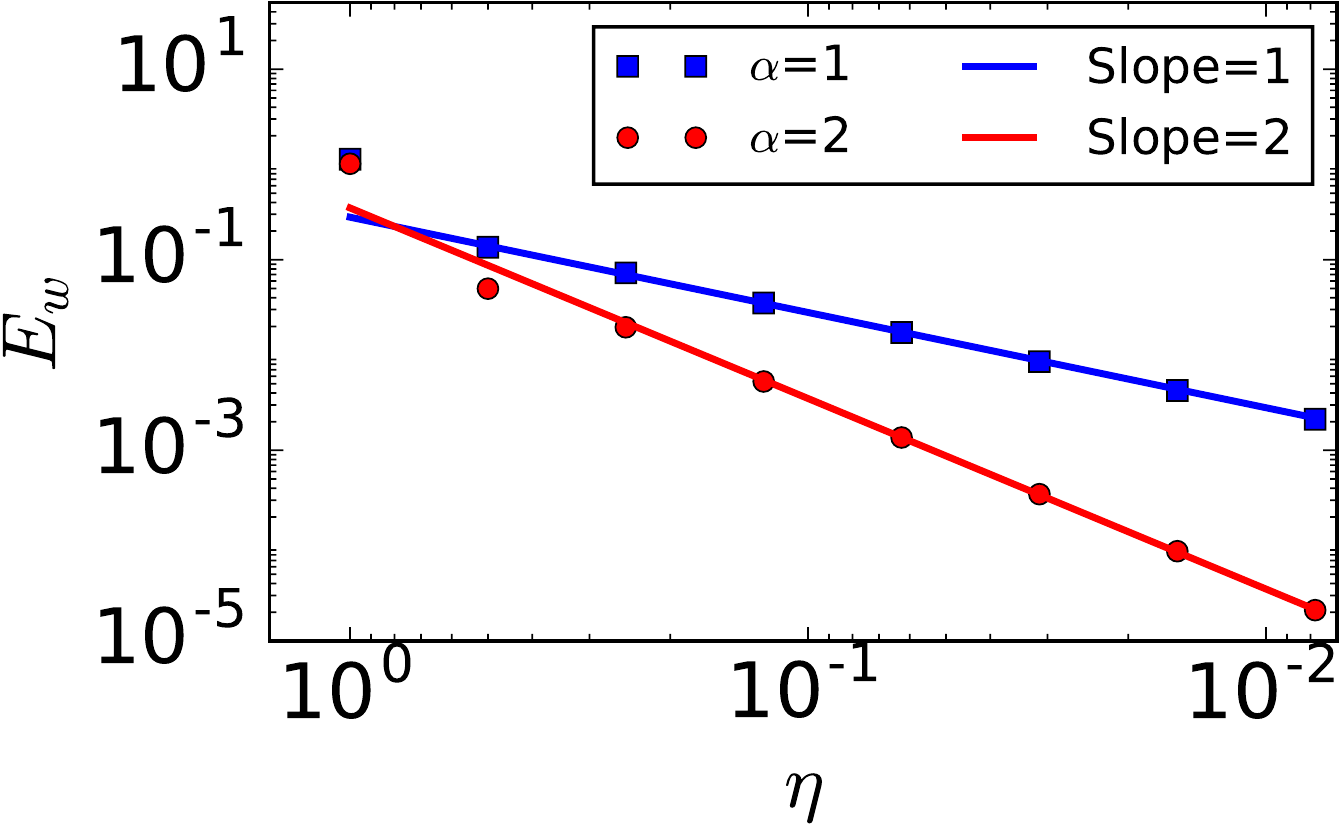}}	
	\caption{Weak error $E_w$, as defined in~\eqref{eq:weak_err} with $\alpha=1,2$, v.s. learning rate $\eta$ for two different choices of $f,f_i$. All errors are averaged over $10^{12}$ samples of SGD trajectories up to $T=1.0$. The initial condition is $x_0=1$. The SMEs moments are solved exactly since they involve linear drifts. (a) Quadratic objective with $n=2$, $f_i=(x-\gamma_i)^2$ where $\gamma_i\in\{\pm 1\}$. The total objective is $f(x)=x^2+1$. The test function is $g(x)=x+x^2+x^3$. (b) Non-convex $f_i$'s with $n=2$, $f_i(x)=(x-\gamma_i)^2 + \gamma_i x^3$ where $\gamma_i\in\{\pm 1\}$. The total objective is the same $f(x)=x^2+1$. We chose $g(x)=x$ so that $\E g(X_T)$ has closed form solution. Note that for this choice of $f_i$, the condition (iii) of Theorem~\ref{thm:sme} is not satisfied. Nevertheless, in both cases, we observe that the weak error decreases with $\eta$ like $E_w\sim\eta^\alpha$. }
	\label{fig:error_smg_sgd}
\end{figure}

\begin{remark}
	From above, we also observe that if we pick $b(x)=-\nabla f(x)$ and $\sigma(x)$ to be {\it any} function in $G$ (and its sufficiently high derivatives are also in $G$), then we have matching moments up to order $\eta^2$ and hence we can conclude that for this choice, the resulting SDE is a first order weak approximation of the SGD,
	with $\vert \E g(X_{k\eta}) - \E g(x_k)  \vert \leq C\eta \nonumber$, $k=0,1,\dots,N$. In particular, the deterministic gradient flow is a first order weak approximation of the SGD. Hence, just like traditional modified equations, our SME~\eqref{eq:sme} ($\alpha=2$) is the next order approximation of the underlying algorithm.
	
	However, we stress that for our first order SME with $\alpha=1$, i.e. the choice $b=-\nabla f$ and $\sigma=\Sigma^\frac{1}{2}$, the fact that we did not the improve the order of weak convergence from the deterministic gradient flow does not mean that this is a equally bad approximation. The constant $C$ in the weak error depends on the choice of $\Sigma$ and in fact, it can be shown empirically that with this choice, we do have lower weak error $\vert \E g(X_{k\eta}) - \E g(x_k)  \vert $, but the order of convergence of the weak error as $\eta\rightarrow 0$ is the same. An analytical justification must then rely on using  the It\^{o}-Taylor expansion to obtain precise estimates for the factor $C$ (see e.g.~\cite{talay1990expansion}). This is beyond the scope of the current paper.
\end{remark}

\begin{remark}
	The Lipschitz condition (i) is to ensure that the SME has a unique strong solution with uniformly bounded moments~\cite{mil1986weak}. If we allow weak solutions and establish uniform boundedness of moments by other means (more assumptions on the growth and direction of $\nabla f$ for large $x$), then condition (i) is expected to be relaxed although the technical details will be tedious. 
	
	Condition (iii) in Theorem~\ref{thm:sme} appears to be the most stringent one and in fact it may limit applications to problems with objectives that have more than quadratic growth. However, closer inspection tells us it can also be relaxed. For example, if there exists an invariant distribution that concentrates on a compact subset of $\R^d$ then as $\eta\rightarrow 0$, $x_k$'s would be bounded with high probability, and hence for large $x$ we may replace $f,f_i$ with a version that satisfies the growth condition in (iii). Further work is needed to make this precise but we can already see in Fig.~\ref{fig:error_smg_sgd}(b) that we have quadratic weak convergence even when (iii) is not satisfied.  
\end{remark}
\begin{remark}
	The regularity conditions on $f$ and $g$ in Theorem~\ref{thm:sme} are inherited from Theorem~\ref{app:thm:milstein} in~\cite{mil1986weak}. For smooth objectives, polynomial growth conditions are usually not restrictive. Still, with care, these should be relaxed since in our case the small noise helps to reduce the number of terms containing higher derivatives in various Taylor and It\^{o}-Taylor expansions. Proving a more general version of Theorem~\ref{thm:sme} will be left as future work. 
\end{remark}

\section{Derivation of SMEs} 
\label{app:sec:SME_derivation}
In this section, we include more detailed derivations of the SMEs used in the main paper. For brevity, we do not include rigorous proofs of approximation statements for SGD variants in Sec.~\ref{app:sec:SME_derivation_lr} and~\ref{app:sec:SME_derivation_mom}, but only heuristic justifications. Proving rigorous statements for these approximations can be done by modifying the proof of Thm. 1. 

\subsection{SME for the simple quadratic example}
We start with the example in Sec. 3.1 of the main paper. Let $n=2$, $d=1$ and set $f(x)=x^2+1$ with $f_1(x)=(x-1)^2$ and $f_2(x)=(x+1)^2$. The SGD iterations picks at random between $f_1$ and $f_2$ and performs descent with their respective gradients. Recall that the (second order) SME is given by
\begin{equation}
dX_t = - (f^\prime(X_t)+\frac{1}{2}\eta f^\prime(X_t)f^{\prime\prime}(X_t)) dt + (\eta \Sigma(X_t))^\frac{1}{2} dW_t.
\end{equation}
Now, $f^\prime(x) = 2x$, $f^{\prime\prime}(x)=2$ and 
\begin{equation}
\Sigma(x) = \frac{1}{2}\sum_{i=1}^{2} (f_i^\prime(x)-f_i(x))^{2} = 4
\end{equation}
and hence the SME is 
\begin{equation}
dX_t = - 2(1+\eta)X_t dt + 2\sqrt{\eta} dW_t.
\end{equation}

\subsection{SME for learning rate adjustment}
\label{app:sec:SME_derivation_lr}
The SGD iterations with learning rate adjustment is 
\begin{equation}
x_{k+1} = x_{k} - \eta u_k f'(x_k),
\end{equation}
where $u_k\in[0,1]$ is the learning rate adjustment factor. $\eta$ is the maximum allowed learning rate. There are two reasons we introduce this hyper-parameter. First, gradients cannot be arbitrarily large since that will cause instabilities. Second, the SME is only an approximation of the SGD for small learning rates, and so it is hard to justify the approximation for large $\eta$. 

In this case, deriving the corresponding SME is extremely simple. Notice that we can define $g_{i,k}(x_k)=u_k f_{i}(x_k)$, $g_k = u_k f(x_k)$. Then, the iterations above is simply
\begin{equation}
x_{k+1} = x_{k} - \eta g'_{\gamma_k,k}(x_k),
\end{equation}
whose (first order) SME is by Thm. 1
\begin{equation}
dX_t = - g'(X_t) dt + (\eta u^2_t \Sigma(X_t))^\frac{1}{2}dW_t.
\end{equation}
And hence the SME for SGD with learning rate adjustments is
\begin{equation}
dX_t = - u_t f'(X_t) dt + u_t (\eta \Sigma(X_t))^\frac{1}{2}dW_t.
\end{equation}

\subsection{SME for SGD with momentum}
\label{app:sec:SME_derivation_mom}
First let us consider the constant momentum parameter case. The SGD with momentum is the paired update
\begin{align}
v_{k+1} &= \mu v_k - \eta f'_{\gamma_k}(x_k), \nonumber\\
x_{k+1} &= x_k + v_{k+1}. 
\end{align}
To derive and SME, notice that we can write the above as
\begin{align}
v_{k+1} &= v_k + \eta \left( -\frac{1-\mu}{\eta} v_k - f'(x_k) \right) + \eta \left(f'(x_k)-f'_{\gamma_k}(x_k)\right), \nonumber\\
x_{k+1} &= x_k + \eta \left(\frac{v_{k+1}}{\eta}\right). 
\end{align}
Recall that since we are looking at first order weak approximations, it is sufficient to compare to the Euler-Maruyama discretization (Sec.~\ref{app:sec:numerical}). We observe that the above can be seen as an Euler-Maruyama discretization of the coupled SDE
\begin{align}
dV_t &= ( -\frac{1-\mu}{\eta} V_t - f'(X_t) )dt + (\eta \Sigma(X_t))^{\frac{1}{2}} dW_t, \nonumber\\
dX_t &= \frac{1}{\eta}V_t dt, 
\end{align}
with the usual choice of $\Sigma(x)$. Hence, this is the first order SME for the SGD with momentum having a constant momentum parameter $\mu$. For time-varying momentum parameter, we just replace $\mu$ by $\mu_t$ to get
\begin{align}
dV_t &= ( -\frac{1-\mu_t}{\eta} V_t - f'(X_t) )dt + (\eta \Sigma(X_t))^{\frac{1}{2}} dW_t, \nonumber\\
dX_t &= \frac{1}{\eta}V_t dt. 
\end{align}

\section{Solution of optimal control problems}
\label{app:sec:opt_control_soln}
\subsection{Brief introduction to optimal control}
We first introduce some basic terminologies and results on optimal control theory to pave way for our solutions to optimal control problems for the learning rate and momentum parameter. For simplicity, we restrict to one state dimension ($d=1$), but similar equations hold for multiple dimensions. For a more thorough introduction to optimal control theory and calculus of variations, we refer the reader to~\cite{liberzon2012calculus}. 

Let $t\in[0,T]$ and consider the ODE
\begin{equation}
\frac{d}{dt} z_t = \Phi(z_t,u_t),
\label{app:eq:control_ode}
\end{equation}
where $z_t,u_t\in\mathbb{R}$ and $\Phi:\mathbb{R}\times\mathbb{R}\rightarrow \mathbb{R}$. The variable $z_t$ describes the evolution of some state and $u_t$ is the control variable, which can affect the state dynamics. Consider the control problem of minimizing the cost functional
\begin{equation}
C(u) = \int_{0}^{T} L(z_s,u_s) ds + G(z_T)
\end{equation}
with respect to $u_t$, subject to $z_t$ satisfying the ODE~\eqref{app:eq:control_ode} with prescribed initial condition $z_0\in\mathbb{R}$. The function $L$ is known as the running cost and $G$ is the terminal cost. Usually, we also specify some control set $U\subset \mathbb{R}$ so that we only consider $u:[0,T]\rightarrow U$. The full control problem reads
\begin{equation}
\min_{u:[0,T]\rightarrow U} C(u) \text{ subject to } \eqref{app:eq:control_ode}. 
\label{app:eq:control_problem}
\end{equation}
Note that additional path constraints can also be added and~\eqref{app:eq:control_ode} can also be made time-inhomogeneous, but for our purposes it is sufficient to consider the above form. 

There are two principal ways of solving optimal control problems: either dynamic programming through the Hamilton-Jacobi-Bellman (HJB) equation~\citep{bellman1956dynamic}, or using the Pontryagin's maximum principle (PMP)~\citep{pontryagin1987mathematical}. In this section, we will only discuss the HJB method as this is the one we employ to solve the relevant control problems in this paper. 

\subsection{Dynamic programming and the HJB equation}
\label{app:sec:HJB}
The first way to solve~\eqref{app:eq:control_problem} is through the dynamic programming principle. For $t\in[0,T]$ and $z\in\mathbb{R}$, define the value function
\begin{align}
V(z,t) &= \min_{u:[t,T]\rightarrow U} \int_{t}^{T} L(z_s,u_s) ds + G(z_T),\nonumber\\
&\text{subject to} \nonumber\\
\frac{d}{ds} z_s &= \Phi(z_s,u_s), \quad s\in[t,T],\nonumber\\
z(t) &=z. 
\end{align}
Notice that if there exists a solution to~\eqref{app:eq:control_problem}, then the value of the minimum cost is $V(z_0,0)$. The dynamic programming principle allows us to derive a recursion on the function $V$, in the form of a partial differential equation (PDE)
\begin{align}
\partial_t V(z,t) + \min_{u\in U} \{\partial_z V(z,t) \Phi(z,u)+L(z,u)\} &= 0, \nonumber\\
V(T,z) &= G(z).
\label{app:eq:HJB}
\end{align}
This is known as the {\it Hamilton-Jacobi-Bellman equation} (HJB). Note that this PDE is solved backwards in time. The derivation of this PDE can be found in most references on optimal control, e.g. in~\cite{liberzon2012calculus}. The main idea is the dynamic programming principle: for any $t$ the $[t,T]$-portion of the optimal trajectory must again be optimal. 

After solving the HJB~\eqref{app:eq:HJB}, we can then obtain the optimal control $u^*_t$ as function of the state process $z_t$ and $t$, given by 
\begin{equation}
u^*_t = \underset{u\in U}{\text{arg min}}\{\partial_z V(z_t,t) \Phi(z_t,u)+L(z_t,u)\}.
\label{app:eq:HJB_opt_control}
\end{equation}
In some cases, we find that the optimal control is independent of time and is strictly of a {\it feed-back control law}, i.e. 
\begin{equation}
u^*_t = u^*(z_t)
\end{equation}
for some function $u^*:\mathbb{R}\rightarrow U$. This is the case for the problems considered in this paper. With the optimal control found in~\eqref{app:eq:HJB_opt_control}, we can then substitute $u_t=u^*_t$ in~\eqref{app:eq:control_ode} to obtain the optimally controlled process $z^*_t$. 

In summary, to solve the optimal control problem~\eqref{app:eq:control_problem}, we first solve the HJB PDE~\eqref{app:eq:HJB}, and then solve for the optimal control~\eqref{app:eq:HJB_opt_control}, and lastly (if necessary) solve the optimally controlled state process by substituting the solution of~\eqref{app:eq:HJB_opt_control} into~\eqref{app:eq:control_ode}. Sometimes, the optimal control~\eqref{app:eq:HJB_opt_control} can be solved without fully solving the HJB for $V$, e.g. when $L=0$ and one can infer the sign of $\partial V$. This is the case for the two control problems we encounter in this paper. The solution to~\eqref{app:eq:HJB_opt_control} is the most important for all practical purposes since it gives a way to adjust the control parameters on-the-fly, especially when we have a feed back control law.  

\subsection{Solution of the learning rate control problem}
\label{app:sec:soln_lr_problem}
Now, let us apply the HJB equations (Sec.~\ref{app:sec:HJB}) to solve the learning rate control problem. Recall from Sec. 4.1.2 that we wish to solve 
\begin{align}
\min_{u:[0,T]\rightarrow[0,1]} & m_T \text{ subject to } \nonumber\\
\frac{d}{dt}m_t &= -2a u_t m_t + \frac{1}{2} a\eta \Sigma u_t^2,\nonumber\\
m_0 &= \frac{1}{2}a(x_0-b)^2.
\label{app:eq:control_lr_ode}
\end{align}
This is of the form~\eqref{app:eq:control_problem} with $\Phi(m,u)=-2aum+a\eta \Sigma u^2 /2$, $L(m,u)=0$ and $G(m)=m$. Thus, we write the HJB equation
\begin{align}
\partial_t V(m,t) + \min_{u\in[0,1]} \{\partial_m V(m,t) [-2aum + \frac{1}{2}a\eta\Sigma u^2 ] \} &= 0 ,\nonumber\\
V(m,T) &= m.
\end{align}
First, it's not hard to see that for $a>0$, $\partial_m V\geq 0$ for all $m,t$. This is because, the lower the $m$, the closer we are to the optimum and hence the minimum cost achievable in the same time interval $[t,T]$ should be less. Similarly,$\partial_m V\geq 0$ holds for $a<0$ if one reverses all previous statements (in this case $m$ is negative). Hence, we can calculate the minimum
\begin{align}
u^* &= \underset{u\in[0,1]}{\text{arg min}} \{-2aum+\frac{1}{2}a\eta\Sigma u^2\},\nonumber\\
&= \begin{cases} 
1 \qquad &a\leq0,\\ 
\min(1,\frac{2 m}{\eta\Sigma})\qquad &a>0. 
\end{cases}
\end{align}
Notice that this solution is a feed-back control policy. We can now substitute $u_t=u^*_t$ where
\begin{align}
u^*_t &= \begin{cases} 
1 \qquad &a\leq0,\\ 
\min(1,\frac{2 m_t}{\eta\Sigma})\qquad &a>0. 
\end{cases}
\label{app:eq:control_lr_u}
\end{align}
into the ODE in ~\eqref{app:eq:control_lr_ode} to obtain
\begin{align}
m^*_t &= \begin{cases} 
m_0 e^{-2 a t} + \frac{1}{4} \eta  \Sigma (1-e^{-2 a t})  \qquad &a\leq 0 \text{ or } t<t^*,\\ 
\frac{\eta\Sigma}{2 + 2 a (t-t^*)} \qquad &a>0 \text{ and } t\geq t^*. 
\end{cases}
\end{align}
where 
\begin{equation}
t^*=\frac{1}{2a}\log\left(\frac{4m_0}{\eta\Sigma}-1\right)
\end{equation}
And therefore, we get from~\eqref{app:eq:control_lr_u} the effective annealing schedule 
\begin{equation}
u^*_t = \begin{cases} 
1 \qquad &a\leq 0\text{ or }t\geq t^*,\\ 
\frac{1}{1+a(t-t^*)}\qquad &a>0\text{ and } t>t^*,
\end{cases}
\label{app:eq:control_annealing_u}
\end{equation}

\subsection{Solution of the momentum parameter control problem}
\label{app:sec:soln_mom_problem}
We shall consider the case $a>0$, since for $a\leq 0$ the optimal control is trivially $\mu_t = 1$. The momentum parameter control problem is 
\begin{align}
\min_{\mu:[0,T]\rightarrow[0,1]} & m_T \text{ subject to } \nonumber\\
\frac{d}{dt}m_t &= \mathcal{R} \lambda(\mu_t) (m_t - m_\infty(\mu_t)),\nonumber\\
m_0 &= \frac{1}{2}a(x_0-b)^2,
\label{app:eq:control_mom_ode}
\end{align}
where
\begin{equation}
\lambda(\mu) = - \frac{(1-\mu)-\sqrt{(1-\mu)^2-4a\eta}}{\eta}, \qquad m_{\infty}(\mu) = \frac{\eta\Sigma}{4(1-\mu)}.
\end{equation}
This is of the form~\eqref{app:eq:control_problem} with $\Phi(m,\mu)=\mathcal{R}\lambda(\mu)(m-m_\infty(\mu))$, $L(m,u)=0$ and $G(m)=m$. The HJB equation is
\begin{align}
\partial_t V(m,t) + \min_{\mu\in[0,1]} \{\partial_m V(m,t) [\mathcal{R}\lambda(\mu)(m-m_\infty(\mu))] \} &= 0 ,\nonumber\\
V(m,T) &= m.
\end{align}
Again, it is easy to see that $\partial_m V(m,t)\geq 0$ for all $m,t$ and so 
\begin{equation}
\mu^* = \underset{\mu\in[0,1]}{\text{arg min}} \{\mathcal{R}\lambda(\mu)(m-m_\infty(\mu)) \}
\label{app:eq:control_mom_minprob}
\end{equation}
This minimization problem has no closed form solution. However, observe that $\mathcal{R}\lambda(\mu)\leq 0$ and is minimized at $\mu = \mu_{\text{opt}} = \max(0,1-2\sqrt{a\eta})$. Now, if $\mu>\mu_\text{opt}$, we have $\mathcal{R}\lambda(\mu)=-(1-\mu)/\eta$ and so $\mathcal{R}\lambda(\mu)(m-m_\infty(\mu))$ is increasing in $\mu$ for $\mu>\mu_\text{opt}$ (one can check this by differentiation and showing that the derivative is always positive). Hence, $\mu^*\leq \mu_\text{opt}$ and it is sufficient to consider $\mu\in[0,\mu_\text{opt}]$ in the minimization problem~\eqref{app:eq:control_mom_minprob}. 

Next, observe that $m-m_\infty(\mu)$ is decreasing in $\mu$ and negative if 
\begin{equation}
m < \frac{\eta\Sigma}{4(1-\mu)}.
\end{equation}
or
\begin{equation}
\mu > 1-\frac{\eta\Sigma}{4m}.
\end{equation}
At the same time, $\mathcal{R}\lambda(\mu)$ is negative and decreasing for $\mu\in[0,\mu_\text{opt}]$. Thus, the product $\mathcal{R}\lambda(\mu)(m-m_\infty(\mu))$ is positive and increasing for $1-\frac{\eta\Sigma}{4m}<\mu<\mu_\text{opt}$ and hence we must have 
\begin{equation}
\mu^* \leq 1-\frac{\eta\Sigma}{4m}. 
\end{equation}
Note that this is only a bound, but for small $\eta$, we can take this as an approximation of $\mu^*$, so long as it is less than $\mu_\text{opt}$. Hence, we arrive at
\begin{equation}
\mu^*_t = \begin{cases} 
1 \qquad &a\leq0,\\ 
\min(\mu_{\text{opt}},\max(0,1-\frac{\eta\Sigma}{4m_t}))\qquad &a>0.
\end{cases}
\label{app:eq:control_mom_mu}
\end{equation}
One can of course follow the steps in Sec.~\ref{app:sec:soln_lr_problem} to calculate $m^*_t$ and hence $\mu^*_t$ in the form of an annealing schedule. We omit these calculations since they are not relevant to applications.

\section{Numerical experiments}
\label{app:sec:implementation}

In this section, we provide model and algorithmic details for the various numerical experiments considered in the main paper, as well as a brief description of the commonly applied adaptive learning rate methods that we compare the cSGD algorithm with.

\subsection{Model details}
\label{app:sec:model_details}
In Sec. 4 from the main paper, we consider three separate models for two datasets. 
\subsubsection*{M0: fully connected NN on MNIST}
The first dataset we consider the MNIST dataset~\citep{lecun1998mnist}, which involves computer recognition of 60000 $28\times 28$ pixel pictures of handwritten digits. We split it into 55000 training samples and 5000 test samples. Our inference model is a fully connected neural network with one hidden layer. For a input batch $K$ of pixel data (flattened into a vector) $z\in\mathbb{R}^{784\times K}$, we define the model
\begin{equation}
y = \text{softmax}(W_2 h_R(W_1 z + b_1) + b2),
\end{equation}
where the activation function $h_R$ is the commonly used Rectified Linear Unit (ReLU) 
\begin{equation}
h_R(z)_{(ij)} = \max(z_{(ij)},0).
\end{equation}
The first layer weights and biases are $W_1\in\mathbb{R}^{784\times10}$ and $b_1\in\mathbb{R}^{10}$ and the second layer weights and biases are $W_2\in\mathbb{R}^{10\times10}$ and $b_2\in\mathbb{R}^{10}$. These constitute the trainable parameters. 
The $\text{softmax}$ function is defined as
\begin{equation}
\text{softmax}(z)_{(ij)} = \frac{\exp(-z_{(ij)})}{\sum_k \exp(-z_{(kj)})}.
\end{equation}
The output tensors $y\in\mathbb{R}^{10\times K}$ is compared to a batch of one-hot target labels $\hat{y}$ with the cross-entropy loss
\begin{equation}
C(y,\hat{y}) = -\frac{1}{10K}\sum_{i,j} \hat{y}_{(ij)}\log y_{(ij)}.
\end{equation}
Lastly, we use $\ell_2$ regularization so that the minimization problem is
\begin{equation}
\min_{W_1,b_1,W_2,b_2} C(y,\hat{y}) + \sum_{i=1}^{2}\lambda_{W,i}\Vert W_i\Vert_2^2 + \sum_{i=1}^{2}\lambda_{b,i}\Vert b_i\Vert_2^2,
\end{equation}
Each regularization strength $\lambda$ is set to be 1 divided by the dimension of the trainable parameter. 
\subsubsection*{C0: fully connected NN on CIFAR-10}
The CIFAR-10 dataset~\citep{krizhevsky2009learning} consists of 60000 small $32\times 32$ pixels of RGB natural images belonging to ten separate classes. We split the dataset into 50000 training samples and 10000 test samples. 
Our first model for this dataset is a deeper fully connected neural network
\begin{equation}
y = \text{softmax}(W_3h_T(W_2h_T(W_1z + b_1)+b_2) + b_3),
\end{equation}
where we use a tanh activation function between the hidden layers 
\begin{equation}
h_T(z)_{(ij)} = \text{tanh}(z_{(ij)}).
\end{equation}
The layers have width 3071,500,300,10. That is, the trainable parameters have dimensions $W_1\in\mathbb{R}^{3071\times 500}, b_1\in\mathbb{R}^{500}$, $W_2\in\mathbb{R}^{500\times 300}, b_2\in\mathbb{R}^{300}$, $W_3\in\mathbb{R}^{300\times 10}, b_3\in\mathbb{R}^{10}$. We use the same soft-max output, cross-entropy loss and $\ell_2$ regularization as as before. 
\subsubsection*{C1: convolutional NN on CIFAR-10}
Our last experiment is a convolutional neural network on the same CIFAR-10 dataset. We use four convolution layers consisting of convolution,batch-normalization,ReLU,max-pooling. Convolution filter size is $5\times 5$, with uniform stride 1 and padding 2. Output channels of convolution layers are \{96,128,256,64\}. The pooling size is $2\times 2$ with stride $2$. The output layers consist of two fully connected layers of width \{1024,256\} and drop-out rate $0.5$. $\ell_2$ regularization is introduced as a weight decay with decay parameter 5e-3. 

\subsection{Adagrad and Adam}
\label{app:sec:ada_methods}
Here, we write down for completeness the iteration rules of Adagrad~\citep{duchi2011adaptive}, and Adam~\citep{kingma2015adam} optimizers, which are commonly applied tools to tune the learning rate. For more details and background, the reader should consult the respective references. 

\textbf{Adagrad.} The Adagrad modification to the SGD reads 
\begin{equation}
x_{(i),k+1} = x_{k,(i)} - \frac{\eta}{\sqrt{G_{k,(i)}}}\partial_{(i)}f_{\gamma_k}(x_k),
\end{equation}
where $G_{k,(i)}$ is the running sum of gradients $\partial_{(i)}f_{\gamma_l}(x_l)$ for $l=0,\dots,k-1$. The tunable hyper-parameters are the learning rate $\eta$ and the initial accumulator value $G_0$. In this paper we consider only the learning rate hyper-parameter as this is equivalent to setting the initial accumulator to a common constant across all dimensions. 


\textbf{Adam.} The Adam method has similar ideas to momentum. It keeps the exponential moving averages
\begin{align}
m_{(i),k+1} &= \beta_1 m_{k,(i)} + (1-\beta_1) \partial_{(i)}f_{\gamma_k}(x_k),\nonumber\\
v_{(i),k+1} &= \beta_2 v_{k,(i)} + (1-\beta_2) [\partial_{(i)}f_{\gamma_k}(x_k)]^2. 
\end{align}
Next, set,
\begin{align}
\hat{m}_{k,(i)} &= \frac{m_{k,(i)}}{1-\beta_1^k},\nonumber\\
\hat{v}_{k,(i)} &= \frac{v_{k,(i)}}{1-\beta_2^k}.
\end{align}
Finally, the Adam update is
\begin{equation}
x_{(i),k+1} = x_{k,(i)} - \frac{\eta}{\sqrt{\hat{v}_{k,(i)}}} \hat{m}_{k,(i)}. 
\end{equation}
The hyper-parameters are the learning rate $\eta$ and the EMA decay parameters $\beta_1,\beta_2$. 

Note that for both methods above, one can also introduce a regularization term $\epsilon$ to the denominator to prevent numerical instabilities. 

\subsection{Implementation of cSGD}
\label{app:sec:csgd}
Recall from Sec. 4.1 that the optimal control solution for learning rate control of the quadratic objective $f(x)=\frac{1}{2}a(x-b)^2$ is given by
\begin{align}
u^*_t &= \begin{cases} 
1 \qquad &a\leq0,\\ 
\min(1,\frac{2 m_t}{\eta\Sigma})\qquad &a>0. 
\end{cases}
\label{app:eq:sme_lr_exact_policy}
\end{align}
The idea is to perform a local quadratic approximation
\begin{equation}
f(x) \approx \frac{1}{2} \sum_{i=1}^{d} a_{(i)}(x_{(i)}-b_{(i)})^2.
\end{equation}
This is equivalent to a local linear approximation of the gradient, i.e. for $i=1,2,\dots,d$
\begin{equation}
\partial_{(i)} f(x) \approx a_{(i)} ( x_{(i)} - b_{(i)} ).
\end{equation}
This effectively decouples the control problems of $d$ identical one-dimensional control problems, so that we may apply~\eqref{app:eq:sme_lr_exact_policy} element-wise. 
We note that this approximation is only assumed to hold locally and the parameters must be updated. There are many ways to do this. Our approach uses linear regression on-the-fly via exponential moving averages (EMA). For each trainable dimension $i$, we maintain the following exponential averages
\begin{align}
\overline{g}_{k+1,(i)} &= \beta_{k,(i)} \overline{g}_{k,(i)} + (1-\beta_{k,(i)}) f'_{\gamma_{k}}(x_{k,(i)}),\nonumber\\
\overline{g^2}_{k+1,(i)} &= \beta_{k,(i)} \overline{g^2}_{k,(i)} + (1-\beta_{k,(i)}) f'_{\gamma_{k}}(x_{k,(i)})^2,\nonumber\\
\overline{x}_{k+1,(i)} &= \beta_{k,(i)} \overline{x}_{k,(i)} + (1-\beta_{k,(i)}) x_{k,(i)},\nonumber\\
\overline{x^2}_{k+1,(i)} &= \beta_{k,(i)} \overline{x^2}_{k,(i)} + (1-\beta_{k,(i)}) x_{k,(i)}^2,\nonumber\\
\overline{gx}_{k+1,(i)} &= \beta_{k,(i)} \overline{gx}_{k,(i)} + (1-\beta_{k,(i)}) x_{k,(i)} f'_{\gamma_{k}}(x_{k,(i)}).
\label{app:eq:lr_running_avg_1}
\end{align}
The decay parameter $\beta_{k,(i)}$ controls the effective averaging window size. In practice, we should adjust $\beta_{k,(i)}$ so that it is small when variations are large, and vice versa. This ensures that our local approximations adapts to the changing landscapes. Since local variations is related to the gradient, we use the following heuristic
\begin{equation}
\beta_{k+1,(i)} = \beta_\text{min} + (\beta_\text{max}-\beta_\text{min}) \frac{\overline{g^2}_{k,(i)} - \overline{g}_{k,(i)}^2}{\overline{g^2}_{k,(i)}}.
\label{app:eq:lr_running_avg_2}
\end{equation}
which is similar to the one employed in~\cite{schaul2013no} for maintaining EMAs. The additional clipping to the range $[\beta_\text{min},\beta_\text{max}]$ is to make sure that there are enough samples to calculate meaningful regressions, and at the same time prevent too large decay values where the contribution of new samples vanish. In the applications presented in this paper, we usually set $\beta_\text{min}=0.9$ and $\beta_\text{max}=0.999$, but results are generally insensitive to these values. 

With the EMAs~\eqref{app:eq:lr_running_avg_1}, we compute $a_{k,(i)}$ by the ordinary-least-squares formula and $\Sigma_{k,(i)}$ as the variance of the gradients:
\begin{align}
a_{k,(i)} &= \frac{\overline{gx}_{k,(i)}-\overline{g}_{k,(i)}\overline{x}_{k,(i)}}{\overline{x^2}_{k,(i)}-\overline{x}_{k,(i)}^2}, \nonumber\\
b_{k,(i)} &= \overline{x}_{k,(i)} - \frac{\overline{g}_{k,(i)}}{a_{k,(i)}}, \nonumber\\
\Sigma_{k,(i)} &= {\overline{g^2}_{k,(i)}-\overline{g}_{k,(i)}^2},
\label{app:eq:lr_running_avg_3}
\end{align}
This allows us to estimate the policy~\eqref{app:eq:sme_lr_exact_policy} as
\begin{equation}
u^*_{k,(i)} = \begin{cases} 
1 \qquad &a_{k,(i)}\leq0,\\ 
\min(1,\frac{a_{k,(i)}(\overline{x}_{k,(i)}-b_{k,(i)})^2}{\eta \Sigma_{k,(i)}})\qquad &a_{k,(i)}>0. 
\end{cases}
\label{app:eq:sme_lr_estimated_policy}
\end{equation}
for $i=1,2,\dots,d$. Since our averages are from exponentially averaged sources, we should also update our learning rate policy in the same way:
\begin{equation}
u_{k+1,(i)} = \beta_{k,(i)} u_{k,(i)} + (1-\beta_{k,(i)}) u^*_{k,(i)}
\label{app:eq:sme_lr_update}
\end{equation}
The algorithm is summarized in Alg.~\ref{app:alg:csgd}
\begin{algorithm}[bt!]
	\caption{controlled SGD (cSGD)}
	\label{app:alg:csgd}
	\begin{algorithmic}
		\STATE {\bfseries Hyper-parameters:} $\eta$, $u_0$
		\STATE Initialize $x_0$; $\beta_{0,(i)}=0.9$ $\forall i$
		\FOR{$k=0$ {\bfseries to} $(\#\text{iterations}-1)$}
		\STATE Compute sample gradient $\nabla f_{\gamma_k}(x_k)$
		\FOR{$i=1$ {\bfseries to} $d$}
		\STATE Update EMA using~\eqref{app:eq:lr_running_avg_1} 
		\STATE Compute $a_{k,(i)}$, $b_{k,(i)}$, $\Sigma_{k,(i)}$ using~\eqref{app:eq:lr_running_avg_3}
		\STATE Compute $u^*_{k,(i)}$ using~\eqref{app:eq:sme_lr_estimated_policy}
		\STATE $\beta_{k+1,(i)} = {(\overline{g^2}_{k,(i)} - \overline{g}_{k,(i)}^2)}/{\overline{g^2}_{k,(i)}}$ and clip
		\STATE $u_{k+1,(i)} = \beta_{k,(i)} u_{k,(i)} + (1-\beta_{k,(i)}) u^*_{k,(i)}$
		\STATE $x_{k+1,(i)} = x_{k,(i)} - \eta u_{k,(i)} \nabla f_{\gamma_k}(x_k)_{(i)}$
		\ENDFOR
		\ENDFOR
	\end{algorithmic}
\end{algorithm} 

\subsection{Implementation of cMSGD}
\label{app:sec:cmsgd}
We wish to apply the momentum parameter control 
\begin{equation}
\mu^*_t = \begin{cases} 
1 \qquad &a\leq0,\\ 
\min(\mu_{\text{opt}},\max(0,1-\frac{\eta\Sigma}{4m_t}))\qquad &a>0,
\end{cases}
\label{app:eq:sme_mom_exact_policy}
\end{equation}
where $\mu_\text{opt} = \max\{0,1-2\sqrt{a\eta}\}$. We proceed in the same way as in Sec.~\ref{app:sec:csgd} by keeping the relevant EMA averages and performing linear regression on the fly. The only difference is the application of the momentum parameter adjustment, which is 
\begin{equation}
\mu^*_{k,(i)} = \begin{cases} 
1 \qquad &a_{k,(i)}\leq0,\\ 
\min[\max(0,1-2\sqrt{a_{k,(i)}\eta}),\\ \max(0,1-\frac{\eta\Sigma_{k,(i)}}{2 a_{k,(i)} (x_{k,(i)}-b_{k,(i)})^2}) ]\qquad &a_{k,(i)}>0,
\end{cases}
\label{app:eq:sme_mom_estimated_policy}
\end{equation}
The algorithm is summarized in Alg.~\ref{app:alg:cmsgd}. 

\begin{algorithm}[bt!]
	\caption{controlled momentum SGD (cMSGD)}
	\label{app:alg:cmsgd}
	\begin{algorithmic}
		\STATE {\bfseries Hyper-parameters:} $\eta$, $\mu_0$
		\STATE Initialize $x_0$, $v_0$; $\beta_{0,(i)}=0.9$ $\forall i$
		\FOR{$k=0$ {\bfseries to} $(\#\text{iterations}-1)$}
		\STATE Compute sample gradient $\nabla f_{\gamma_k}(x_k)$
		\FOR{$i=1$ {\bfseries to} $d$}
		\STATE Update EMA using~\eqref{app:eq:lr_running_avg_1} 
		\STATE Compute $a_{k,(i)}$, $b_{k,(i)}$, $\Sigma_{k,(i)}$ using~\eqref{app:eq:lr_running_avg_3}
		\STATE Compute $\mu^*_{k,(i)}$ using~\eqref{app:eq:sme_mom_estimated_policy}
		\STATE $\beta_{k+1,(i)} = {(\overline{g^2}_{k,(i)} - \overline{g}_{k,(i)}^2)}/{\overline{g^2}_{k,(i)}}$ and clip
		\STATE $\mu_{k+1,(i)} = \beta_{k,(i)} \mu_{k,(i)} + (1-\beta_{k,(i)}) \mu^*_{k,(i)}$
		\STATE $v_{k+1,(i)} = \mu_{k,(i)} v_{k,(i)} - \eta  \nabla f_{\gamma_k}(x_k)_{(i)}$
		\STATE $x_{k+1,(i)} = x_{k,(i)} + v_{k+1,(i)}$
		\ENDFOR
		\ENDFOR
	\end{algorithmic}
\end{algorithm}

\subsection{Training accuracy for C1}

For completeness we also provide in Fig.~\ref{app:fig:train_acc} the training accuracies of C1 with various hyper-parameter choices and methods tested in this work. These complements the plots of test accuracies in Fig. 3,5,6 in the main paper. We see that cSGD and cMSGD display the same robustness in terms of test and training accuracies. 
\begin{figure}[t]
	\centering
	\subfloat[C1, Learning rate adjustments (c.f. main paper Fig . 3)]{\includegraphics[width=12cm]{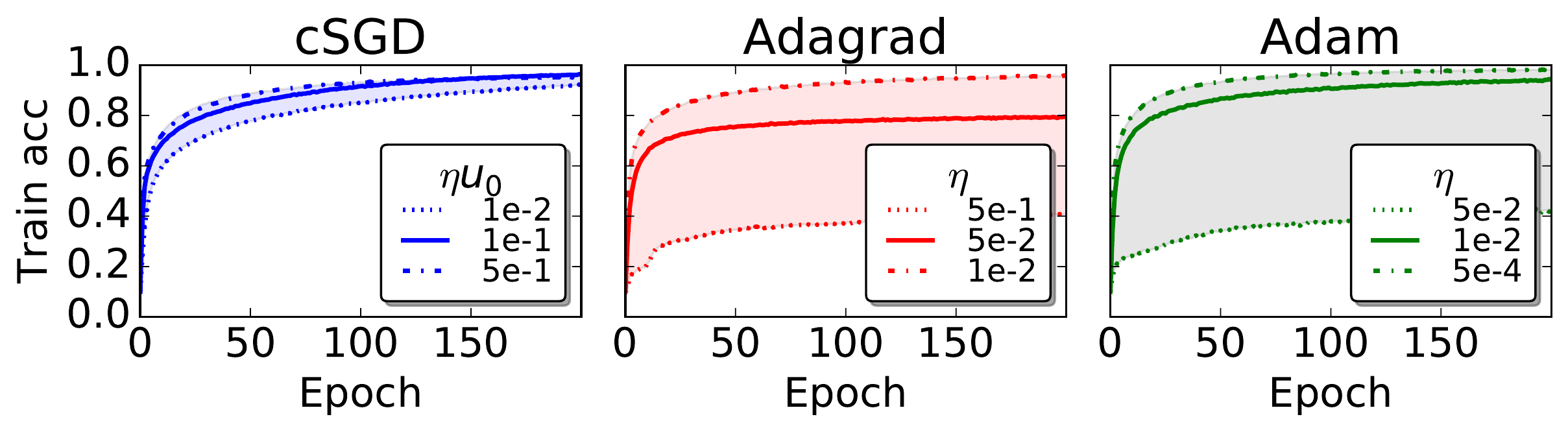}}
	
	\subfloat[C1, Momentum adjustments (c.f. main paper Fig . 5)]{\includegraphics[width=12cm]{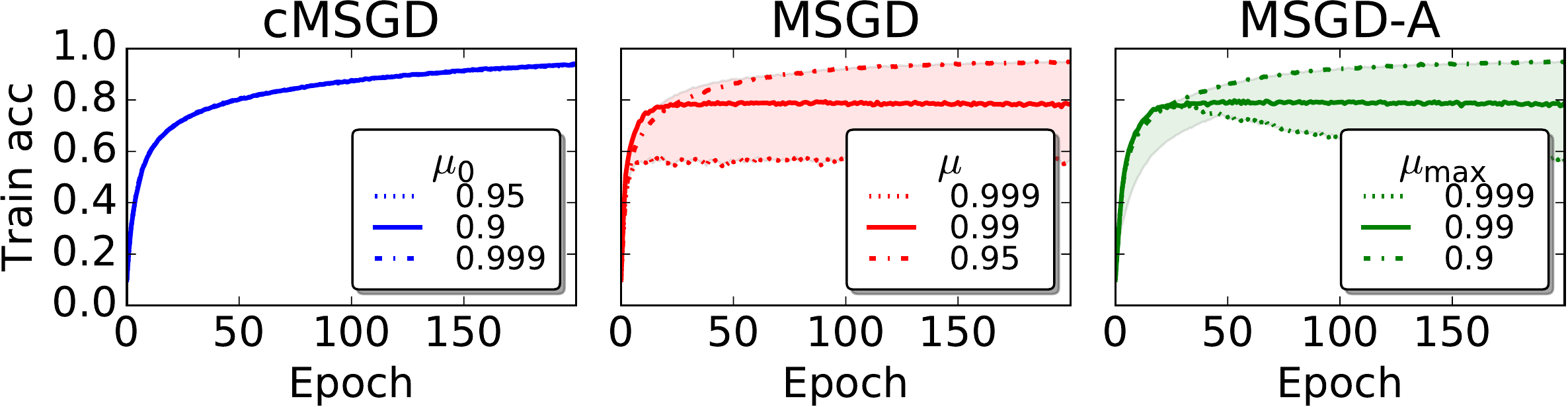}}
	
	\subfloat[C1, Learning rate sensitivity (c.f. main paper Fig . 6)]{\includegraphics[width=12cm]{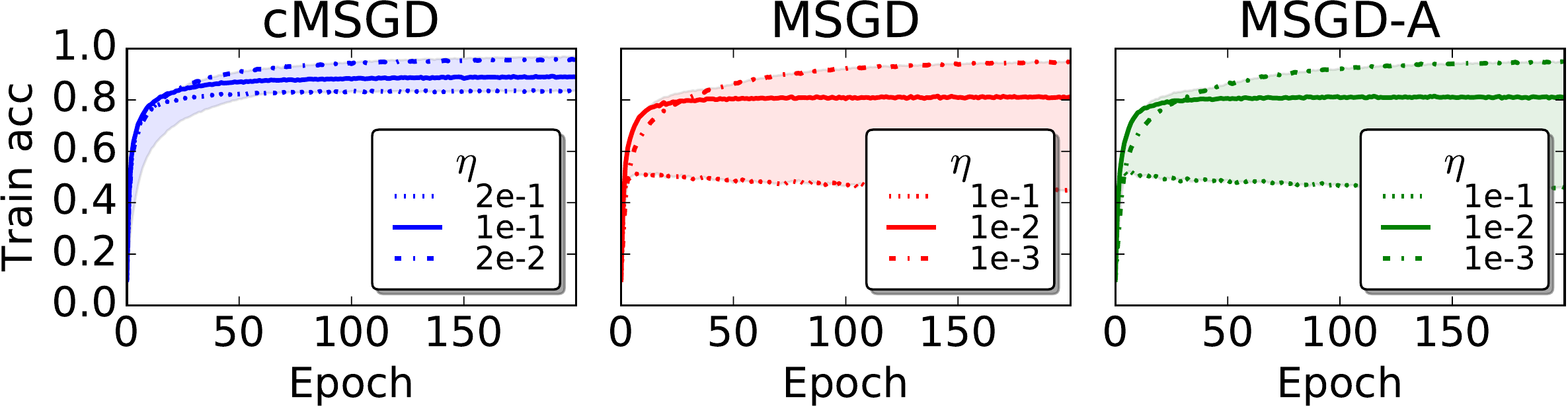}}
	\caption{Training accuracies for various methods and hyper-parameter choices. The set-up is the same as in the main paper, Fig. 3,5,6 except that we plot training accuracy instead of test accuracy. The qualitative observation is the same: cSGD and cMSGD are generally robust to changing parameters and models. }
	\label{app:fig:train_acc}
\end{figure}

\bibliography{ref}
\bibliographystyle{icml2017}

\end{document}